\documentclass[twocolumn]{autart}    
\usepackage{amsmath,amsfonts,amssymb,amscd,bm}
\usepackage{graphicx, subfigure, color}
\usepackage{epstopdf}
\usepackage{natbib}
\usepackage{mathrsfs}
\usepackage{wrapfig}
\usepackage{arydshln}
\usepackage{algorithm}
\usepackage{algpseudocode}
\usepackage{booktabs}
\usepackage{caption}
\usepackage{makecell}
\usepackage{pifont}%
\newcommand{\cmark}{\ding{51}}%
\newcommand{\xmark}{\ding{55}}%
\newtheorem{remark}{Remark}
\newtheorem{ass}{Assumption}
\newtheorem{definition}{Definition}
\newtheorem{corollary}{Corollary}
\newtheorem{proposition}{Proposition}
\newtheorem{lemma}{Lemma}

\usepackage[colorlinks,linkcolor=blue,citecolor=blue]{hyperref}

\allowdisplaybreaks

\begin{document}
\begin{frontmatter}
\title{A Hybrid Stochastic Gradient Tracking Method for Distributed Online Optimization Over Time-Varying Directed Networks}

\author[]{Xinli Shi}\ead{xinli\_shi@seu.edu.cn},
\author[]{Xingxing Yuan}\ead{220224980@seu.edu.cn},
\author[]{Longkang Zhu}\ead{ 230248643@seu.edu.cn},
\author[]{Guanghui Wen}\ead{ghwen@seu.edu.com}




\begin{abstract}
With the increasing scale and dynamics of data, distributed online optimization has become essential for real-time decision-making in various applications. However, existing algorithms often rely on bounded gradient assumptions and overlook the impact of stochastic gradients, especially in time-varying directed networks. This study proposes a novel Time-Varying Hybrid Stochastic Gradient Tracking algorithm named TV-HSGT, based on hybrid stochastic gradient tracking and variance reduction mechanisms. Specifically, TV-HSGT integrates row-stochastic and column-stochastic communication schemes over time-varying digraphs, eliminating the need for Perron vector estimation or out-degree information. By combining current and recursive stochastic gradients, it effectively reduces gradient variance while accurately tracking global descent directions. Theoretical analysis demonstrates that TV-HSGT can achieve improved bounds on dynamic regret without assuming gradient boundedness. Experimental results on logistic regression tasks confirm the effectiveness of TV-HSGT in dynamic and resource-constrained environments.
\end{abstract}
\end{frontmatter}


\section{Introduction}\label{shosd:sec:introduction}
Distributed optimization has received significant attention and found applications in various fields such as control, signal processing, and machine learning \cite[]{shahrampour2015distributed, nedic2017fast,Shahrampour2017ACC}. It aims to solve a large-scale optimization problem by decomposing it into smaller, more tractable subproblems that can be solved iteratively and in parallel by a network of interconnected agents through communication.
Most traditional works on distributed optimization focus on static problems, making them unsuitable for dynamic tasks arising in real-world applications, such as networked autonomous vehicles, smart grids, and online machine learning, among others \cite[]{Dall2020}.   


Online optimization, which addresses time-varying cost functions, plays a vital role in solving dynamic problems in timely application fields \cite[]{Zinkevich2003,Mairal2009,Li2024TAC,Cao2021TAC}. In many practical scenarios, such as machine learning with information streams \cite[]{shalev2012online}, the objective functions of optimization problems change over time, making them inherently dynamic \cite[]{wei2023distributed,Zinkevich2003}. 
Online learning has emerged as a powerful method for handling sequential decision-making tasks in dynamic contexts, enabling real-time operation while ensuring bounded performance loss in terms of regret~\cite[]{hazan2016introduction}. Regret is the gap between the cumulative objective value achieved by the online algorithm and that of the optimal offline solution \cite[]{li2020distributed,Shahrampour2018}. In the literature, two types of regret are commonly considered, i.e., \textit{static} and \textit{dynamic regret}. The former evaluates the performance of an online algorithm relative to a fixed optimal decision $x^*$, and is typically formulated as
$\min_{t=1}^T(f_t(x_t)-f_t(x^*))$, where $x_t$ denotes the output of the online algorithm and $x^*$ is the optimal fixed decision in hindsight, i.e., $x^* \in \arg\min_{t=1}^Tf_t(x)$. In contrast, the dynamic regret is obtained by replacing the above static $x^*$ by a dynamic solution $x_t^*\in \arg\min f_t(x)$. This makes dynamic regret more suitable for non-stationary environments, although it is generally more challenging to minimize due to the evolving nature of the optimal points. Both metrics are commonly used to assess the performance of online algorithms. Achieving a sublinear regret growth, i.e., one that grows slower than linearly with time, is often regarded as a key indicator of algorithmic efficiency~\cite[]{yuan2017adaptive}. Therefore, minimizing regret, particularly in terms of establishing sublinear regret bounds, is fundamental to the design and analysis of effective online optimization methods.

\begin{table*}[t]
\centering
\captionsetup[table]{justification=centering, labelsep=colon}
\caption{Comparison with Distributed Online Optimization Algorithms}
\label{table:comparison}
\renewcommand{\arraystretch}{1.1}
\begin{tabular}{p{5.2cm} >{\centering\arraybackslash}p{3.2cm} cccccc}
\toprule
Works & Weight Matrix & TVN? & SG? & NBG? & Mo. Term? & Regret Type \\
\midrule
\cite{Shahrampour2018}         & Undirected, DS  & \xmark  & \cmark & \xmark & \xmark  & Dynamic \\
\cite{cao2023decentralized}&
Undirected, DS &  \xmark &  \xmark 
&  \xmark &  \xmark & Static \\
\cite{Zhang2022SMC}         & Directed, DS    & \cmark & \xmark  & \xmark & \xmark  & Static  \\
\cite{nazari2022dadam} & Undirected, DS  & \xmark  & \cmark & \xmark & \cmark  & Dynamic \\
\cite{Li2022TCMS}         & Directed, DS    & \cmark & \cmark & \xmark & \xmark  & Dynamic \\
\cite{carnevale2022gtadam} & Undirected, DS & \xmark  & \xmark  & \cmark  & \cmark & Dynamic \\
\cite{Sharma2024TSP}        & Undirected, DS  & \xmark  & \xmark  & \cmark  & \xmark & Dynamic \\
\cite{Li2024TAC}        & Directed, RS    & \xmark  & \cmark & \xmark & \xmark  & Static  \\
\cite{yao2025online} &Directed, RCS 
& \xmark & \xmark & \xmark & \xmark &Dynamic \\
Ours                  & Directed, RCS  & \cmark & \cmark & \cmark  & \cmark & Dynamic  \\
\bottomrule
\end{tabular}
\vspace{1em}
\begin{minipage}{0.9\textwidth}
\small
\textbf{Note:} “TVN?” indicates whether the network is time-varying;  
“SG?” indicates whether stochastic gradients are used;  
“NBG?” means no bounded gradient assumption in the analysis;  
“Mo. Term?” indicates whether momentum terms are incorporated in the algorithm;  
“RCS” stands for row- and column-stochastic weight matrices;  
“RS” and “DS” denote row-stochastic and doubly stochastic weight matrices, respectively;  
“Mo.” is short for momentum.
\end{minipage}
\end{table*}

Distributed online optimization offers a flexible framework for handling dynamic settings, combining the benefits of decentralized computation with the ability to adapt to non-stationary environments. 
Earlier works \cite[]{hosseini2013online, yan2012distributed} investigate online distributed optimization in networks with doubly stochastic mixing matrices and achieve a static regret bound of $\mathcal{O}(\sqrt{T})$. \cite{Shahrampour2018} further consider dynamic regret for both determined and stochastic online distributed optimization. 
\cite{carnevale2022gtadam} propose GTAdam without the bounded gradient assumption, combining gradient tracking and adaptive momentum. However, these works assume static or undirected communication topologies, which are insufficient for modeling dynamic networked systems with directional and time-varying interactions. To address this, several algorithms have been developed under time-varying directed graphs with corresponding theoretical guarantees. For instance, \cite{Lee2018TCNS} propose the ODA-PS algorithm by integrating dual averaging with the Push-Sum protocol over a directed time-varying network, achieving an $\mathcal{O}(\sqrt{T})$ static regret. \cite{Li2021TAC} further extend the Push-Sum framework to handle inequality-constrained optimization over unbalanced networks, establishing sublinear dynamic regret and constraint violation. \cite{Xiong2024TNSE} address feedback delays and propose an event-triggered online mirror descent method with regret guarantees. In addition, stochastic gradient methods have been explored to reduce computational costs. \cite{Lee2017TAC} analyze stochastic dual averaging under gradient noise, while \cite{Li2022TCMS} introduce a gradient tracking scheme with aggregation variables, achieving regret bounds under both exact and noisy gradients.

Nevertheless, many of the above methods rely on the assumption of uniformly bounded gradients and neglect the high variance commonly encountered in practice. Moreover, few of them \cite[]{nazari2022dadam,Lee2017TAC,Li2022TCMS,Li2024TAC} incorporate variance reduction techniques, limiting both accuracy and stability in stochastic settings. To overcome these limitations, recent studies have focused on gradient tracking-based approaches, which aim to approximate global descent directions by dynamically aggregating local gradient information. \cite{Zhang2019CDC} establish dynamic regret bounds for a basic tracking scheme, while \cite{carnevale2022gtadam} propose a momentum-enhanced variant inspired by adaptive methods. \cite{Sharma2024TSP} develop a generalized framework for strongly convex objectives without requiring gradient boundedness, further advancing the applicability of gradient tracking in decentralized online settings.

This work addresses the distributed online stochastic optimization over time-varying directed networks under limited computational resources, where agents interact over asymmetric communication links modeled by time-varying row- and column-stochastic mixing matrices. To overcome the challenges introduced by stochastic gradient noise and dynamic topologies, we design a novel online algorithm that incorporates hybrid variance reduction, gradient tracking, and an AB communication scheme \cite[]{Saadatniaki2020TAC,Pu2021TAC,Nguyen2023}. Table \ref{table:comparison} summarizes the comparison of our methods with several existing online optimization algorithms in terms of communication schemes, gradient assumptions, and types of regret. The main contributions are summarized as follows:

\begin{enumerate}
    \item We propose a Time-Varying Hybrid Stochastic Gradient Tracking method, named by TV-HSGT, for distributed online optimization over dynamic directed networks. It integrates a hybrid variance reduction strategy by combining current and recursive stochastic gradients. This method effectively reduces the variance introduced by stochastic gradients and accelerates convergence, as demonstrated in our experimental results.

    \item To address the limited information access inherent in decentralized systems, the algorithm incorporates a gradient tracking mechanism to approximate the global gradient direction over time-varying directed networks. In addition, an AB communication scheme is employed, utilizing both row-stochastic and column-stochastic weight matrices. This design eliminates the need to estimate the Perron vector, as required in traditional Push-Sum methods, improving practical applicability in directed network settings.

\item The algorithm is implemented within an adapt-then-combine (ATC) framework, which allows for relaxed step-size conditions compared with the combine-then-adapt (CTA) framework \cite[]{li2024npga}. We adopt a dynamic regret metric to evaluate performance and introduce a weighted averaging variable to characterize the deviation between local decisions and the global optimal trajectory. Theoretical analysis establishes upper bounds on dynamic regret, and numerical simulations validate the algorithm’s effectiveness in reducing stochastic gradient variance under dynamic and asymmetric communication topologies.
\end{enumerate}

The remainder of this paper is organized as follows. Section II formulates the problem and introduces necessary notations. Section III provides the proposed TV-HSGT algorithm, and Section IV analyzes its dynamic regret.
Section V presents numerical studies. Finally, we conclude the paper and discuss future directions in Section VI.

\section{PROBLEM FORMULATION}
Consider a networked system composed of $n$ agents, denoted by the set $\mathcal{V} = \{1, 2, \dots, n\}$. The agents communicate through a sequence of time-varying directed graphs $\{\mathcal{G}_t = (\mathcal{V}, \mathcal{E}_t)\}_{t \geq 0}$, where $\mathcal{E}_t \subseteq \mathcal{V} \times \mathcal{V}$ represents the set of available communication links at time $t$. If $(j, i) \in \mathcal{E}_t$, agent $i$ can receive information from agent $j$ at time $t$. This work aims to solve the following distributed online optimization problem:
\begin{equation}
\label{problem}
\min_{x \in \mathbb{R}^d} f_t(x) :=  \frac{1}{n} \sum_{i=1}^{n} f_{i,t}(x), \quad t \geq 0,
\end{equation}
where $x \in \mathbb{R}^d$ is the decision variable, and $f_{i,t}(x): \mathbb{R}^d \rightarrow \mathbb{R}$ denotes the local loss function of agent $i$ at time $t$, defined as the expected loss over a local random variable $\xi_{i,t}$, i.e.,
\(
f_{i,t}(x) := \mathbb{E}_{\xi_{i,t} \sim \mathcal{D}_{i,t}} \left[ \hat{f}_{i,t}(x; \xi_{i,t}) \right],
\)
where $\xi_{i,t}$ is a random variable following the distribution $\mathcal{D}_{i,t}$ at time $t$, and $\hat{f}_{i,t}(x; \xi_{i,t})$ denotes the loss function under the sampled random variable $\xi_{i,t}$.
In practical computation, due to limited computational resources, each agent constructs an unbiased stochastic gradient estimator \(
\nabla \hat{f}_{i,t}(x_{i,t}; \xi_{i,t}),
\) based on the current sample $\xi_{i,t}$, and uses it to update its decision variable. The aim of this study is to design a distributed online optimization algorithm tailored to time-varying directed network topologies, where each agent relies solely on limited computational resources and cooperates with neighbors to effectively minimize $f_t(x)$.

\begin{definition}[Dynamic Regret] \label{def:dynamic_regret}
For a sequence of local decisions $\{x_{i,t}\}$ generated by a given online distributed algorithm, the dynamic regret over $T$ time steps is defined as
\[
R_T^d := \mathbb{E} \left[ \sum_{t=1}^T f_t(\hat{x}_t) - \sum_{t=1}^T f_t(x_t^*) \right],
\]
where $\hat{x}_t := \sum_{i=1}^n [\phi_t]_i x_{i,t}$ denotes a weighted average of all agents’ decisions at time $t$, and $\{x_t^*\}_{t \ge 1}$ denotes the sequence of minimizers of the global objective functions $f_t(x)$. 
\end{definition}

To evaluate the algorithm's performance in a time-varying environment, this work adopts dynamic regret as the performance metric, defined formally in Definition~\ref{def:dynamic_regret}. Dynamic regret quantifies the discrepancy between the cumulative loss of an online algorithm and that of a time-dependent sequence of optimal solutions. Various forms of dynamic regret have been proposed in the literature. In particular, the GTAdam framework~\cite[]{carnevale2022gtadam} considers the version
\(
R_T^d := \mathbb{E} \left[ \sum_{t=1}^T f_t(\bar{x}_t) - \sum_{t=1}^T f_t(x_t^*) \right],
\)
where $\bar{x}_t := \frac{1}{n} \sum_{i=1}^n x_{i,t}$ is the simple average of agents’ decisions. However, GTAdam assumes undirected networks with doubly stochastic weight matrices. In contrast, this work addresses time-varying directed networks, where the weight matrices are not necessarily symmetric or doubly stochastic. Hence, we adopt a weighted average $\hat{x}_t := \sum_{i=1}^n [\phi_t]_i x_{i,t}$, as specified in Definition~\ref{def:dynamic_regret}, where $\phi_t \in \mathbb{R}^n$ is a stochastic vector used to accommodate such network structures.
Compared with static regret, dynamic regret effectively captures the algorithm’s asymptotic behavior relative to the evolving optimal decisions $\{x_t^*\}_{t=1}^T$. 

The time-variability and non-stationarity of the problem are characterized by two regularity measures that reflect changes in the objective functions and the evolving optimal solutions. Specifically, \( q_t \) characterizes the maximum discrepancy between the gradients of local objective functions across agents at two consecutive time steps, while \( p_t \) quantifies the variation between successive optimal solutions. These measures are defined as follows
\begin{align}
& q_t := \sup_{i \in \mathcal{V}} \sup_{x \in \mathbb{R}^d} \left\| \nabla f_{i,t+1}(x) - \nabla f_{i,t}(x) \right\|, \label{qt}\\
& p_t := \left\| x_{t+1}^* - x_t^* \right\| \label{pt}.
\end{align}

We impose the following standard assumptions on the loss functions.
\begin{ass}
\label{ass:strongtu}
The global objective function \(f_t(x)\) is \(\mu\)-strongly convex, i.e., for any \(x, y \in \mathbb{R}^d\), it holds that
\begin{equation}
\langle \nabla f_t(x) - \nabla f_t(y), x - y \rangle \geq \mu \| x - y \|^2,
\end{equation}
where \(\mu > 0\) is the strong convexity parameter.
\end{ass}

\begin{ass}
\label{ass:Lg}
For any agent \(i \in \mathcal{V}\), the stochastic gradient estimator is \(L_g\)-Lipschitz continuous in the mean square sense. That is, for some constant \(L_g > 0\) and any \(x, y \in \mathbb{R}^d\), the following inequality holds 
\begin{equation}
\label{Lipsch}
\mathbb{E} \left[ \|\nabla \hat{f}_{i,t}(x; \xi_{i,t}) - \nabla \hat{f}_{i,t}(y; \xi_{i,t})  \|^2 \right] \leq L_g^2 \|x - y\|^2 .
\end{equation}
\end{ass}

Let~\(\mathcal{F}_t\) denote the \(\sigma\)-algebra generated by \(\{\xi_{i,0}, \xi_{i,1}, \ldots, \xi_{i,t-1}\}\). The following assumption is widely adopted in distributed stochastic optimization and federated learning \cite[]{9226112,pmlr-v139-xin21a,10715643,9713700}.

\begin{ass}
\label{ass:sigma}
For any agent \(i \in \mathcal{V}\), its stochastic gradient is unbiased and has bounded variance, i.e., 
\begin{equation}
\mathbb{E} \left[ \nabla \hat{f}_{i,t}(x, \xi_{i,t}) \mid \mathcal{F}_t \right] = \nabla f_{i,t}(x),
\end{equation}
\begin{equation}
\mathbb{E} \left[ \left\| \nabla \hat{f}_{i,t}(x, \xi_{i,t}) - \nabla f_{i,t}(x) \right\|^2 \mid \mathcal{F}_t \right] \leq \sigma^2,
\end{equation}
where \(\sigma^2 \geq 0\) is a finite constant. 
\end{ass}

Under Assumptions~\ref{ass:Lg} and \ref{ass:sigma}, one can derive that $ f_{i,t}(x)$ is $L_g$-smooth, i.e., 
\begin{equation}
\|\nabla f_{i,t}(x) - \nabla f_{i,t}(y)\| \leq L_g \|x - y\|, \quad \forall x, y \in \mathbb{R}^d.
\end{equation}
Assumptions~\ref{ass:Lg} and \ref{ass:sigma} are standard in establishing the convergence of distributed stochastic optimization algorithms \cite[]{pmlr-v139-xin21a,Huang2024,liu2020optimal,Dinh2022}. 

\section{PROPOSED ALGORITHMS}

In this section, based on an improved stochastic gradient tracking scheme, a novel distributed online optimization algorithm called TV-HSGT is provided to efficiently solve the problem \eqref{problem} over a time-varying directed network.

We define \(\nabla \hat{f}_{i,t+1}(x_{i,t+1}, \xi_{i,t+1})\) and \(\nabla \hat{f}_{i,t+1}(x_{i,t}, \xi_{i,t+1})\) as the stochastic gradients evaluated at \(x_{i,t+1}\) and \(x_{i,t}\), respectively, based on the random sample \(\xi_{i,t+1}\). To reduce the variance inherent in stochastic gradient estimation, we adopt a hybrid variance-reduction approach introduced for stochastic optimization problems \cite[]{liu2020optimal,Dinh2022, pmlr-v139-xin21a}. Let \(z_{i,t}\) denote the hybrid stochastic gradient variable, which is updated as follows
\begin{align}
\label{zchange1}
   z_{i,t+1} & =  (1 - \beta) \left(z_{i,t} - \nabla \hat{f}_{i,t+1}(x_{i,t}, \xi_{i,t+1}) \right) \notag \\
   & \quad +\nabla \hat{f}_{i,t+1}(x_{i,t+1}, \xi_{i,t+1}),
\end{align}
where \(\beta \in [0,1]\) is the mixing parameter. This update rule is equivalent to
\begin{align}
\label{zchange2}
   z_{i,t+1} & = \beta \underbrace{ \nabla \hat{f}_{i,t+1}(x_{i,t+1}, \xi_{i,t+1})}_{\text{stochastic gradient}}  + (1 - \beta) \times \notag \\
   &  \underbrace{ \left(z_{i,t} +  \nabla \hat{f}_{i,t+1}(x_{i,t+1}, \xi_{i,t+1}) - \nabla \hat{f}_{i,t+1}(x_{i,t}, \xi_{i,t+1}) \right)}_{\text{stochastic recursive gradient}}.
\end{align}

When \(\beta = 1\), the method reduces to the standard stochastic gradient, while for \(\beta = 0\), it is equivalent to the stochastic recursive gradient method \cite[]{10.55553305890.3305951}. Compared to classical variance-reduction methods such as SVRG \cite[]{Defazio2014} and SAGA \cite[]{NIPS2013_ac1dd209}, this hybrid strategy offers improved convergence speed and stability\cite[]{pmlr-v139-xin21a}.

While variance reduction enhances gradient estimation stability, each agent in a distributed setting typically only accesses local information, which may not reflect the global objective direction accurately. To address this, the proposed algorithm incorporates a gradient tracking mechanism for estimating the global gradient direction. In contrast to the commonly used CTA framework \cite[]{9226112}, our algorithm employs the ATC framework, which outperforms the CTA framework with larger step-sizes \cite[]{cattivelli2009diffusion,li2024npga}. Each agent $i \in \mathcal{V}$ maintains the variables including the decision variable $x_{i,t} \in \mathbb{R}^d$, the hybrid stochastic gradient variable $z_{i,t} \in \mathbb{R}^d$, and the gradient tracking variable $y_{i,t} \in \mathbb{R}^d$. In each iteration, all agents execute the following procedures in parallel.

Each agent \(i\) sends \(x_{i,t} - \alpha y_{i,t}\) to its out-neighbors \(j \in \mathcal{N}^{\text{out}}_{i,t}\) and receives corresponding vectors from its in-neighbors \(j \in \mathcal{N}^{\text{in}}_{i,t}\), then updates its decision variable as
\begin{align}
\label{xchange1}
x_{i,t+1} = \sum_{j=1}^n [A_t]_{ij} (x_{j,t} - \alpha y_{j,t}),
\end{align}
where \(\alpha > 0\) is the step size, \(\mathcal{N}_{i,t}^{\mathrm{in}}\) and \(\mathcal{N}_{i,t}^{\mathrm{out}}\) denote the in-neighbor and out-neighbor sets of agent \(i\) at time \(t\), respectively.

Next, the agent computes the hybrid stochastic gradient \(z_{i,t+1}\) using \eqref{zchange1}. It then forms the gradient tracking increment \(y_{i,t} + z_{i,t+1} - z_{i,t}\), transmits \([B_t]_{ji}(y_{i,t} + z_{i,t+1} - z_{i,t})\) to each out-neighbor, and updates its gradient tracking variable by
\begin{align}
\label{ychange1}
y_{i,t+1} = \sum_{j=1}^n [B_t]_{ij} \left( y_{j,t} + z_{j,t+1} - z_{j,t} \right).
\end{align}
The detailed execution steps are presented in Algorithm \ref{alg:reactmomentum}.

\begin{algorithm}[t]
\caption{Hybrid Stochastic Gradient Tracking over Time-Varying Directed Networks (TV-HSGT)}
\label{alg:reactmomentum}
\begin{algorithmic}[1]
\Require For each agent \(i \in \mathcal{V}\), initialize \(x_{i,0} \in \mathbb{R}^d\), set \(z_{i,0} = \nabla \hat{f}_{i,0}(x_{i,0}, \xi_{i,0})\), \(y_{i,0} = z_{i,0}\); choose \(\alpha > 0\), \(\beta \in [0,1)\).
\For{\(t = 0,1,2,\ldots,T-1\)}
    \For{each agent \(i = 1, \ldots, n\) in parallel}
        \State Send \(x_{i,t} - \alpha y_{i,t}\) to out-neighbors \(j \in \mathcal{N}^{\text{out}}_{i,t}\), receive corresponding data from in-neighbors.
        \State Update decision variable:
        \[
        x_{i,t+1} = \sum_{j=1}^n [A_t]_{ij} \left(x_{j,t} - \alpha y_{j,t} \right)
        \]
        \State Compute hybrid stochastic gradient:
        \begin{align*}
        z_{i,t+1} &=  (1 - \beta) \left(z_{i,t} - \nabla \hat{f}_{i,t+1}(x_{i,t}, \xi_{i,t+1}) \right) \\
        & \quad +\nabla \hat{f}_{i,t+1}(x_{i,t+1}, \xi_{i,t+1})  
        \end{align*}
        \State Transmit \([B_t]_{ji}(y_{i,t} + z_{i,t+1} - z_{i,t})\) to out-neighbors, receive updates from in-neighbors.
        \State Update gradient tracking variable:
        \[
        y_{i,t+1} = \sum_{j=1}^n [B_t]_{ij} \left( y_{j,t} + z_{j,t+1} - z_{j,t} \right)
        \]
    \EndFor
\EndFor
\end{algorithmic}
\end{algorithm}

The iterative updates rely on two non-negative weight matrices \(A_t\) and \(B_t\), consistent with the structure of the directed graph \(\mathcal{G}_t\). These matrices satisfy
\begin{align*}
[A_t]_{ij} > 0, \, & \forall j \in \mathcal{N}_{i,t}^{\mathrm{in}} \cup \{i\}; \quad [A_t]_{ij} = 0, \, \forall j \notin \mathcal{N}_{i,t}^{\mathrm{in}} \cup \{i\},\\
[B_t]_{ji} > 0, \, & \forall j \in \mathcal{N}_{i,t}^{\mathrm{out}} \cup \{i\}; \quad [B_t]_{ji} = 0, \, \forall j \notin \mathcal{N}_{i,t}^{\mathrm{out}} \cup \{i\}.
\end{align*}

The following introduces the assumptions related to the time-varying communication networks.
\begin{ass}
\label{ass:strong-connectivity}
For any $t \ge 0$, the directed graph \(\mathcal{G}_t\) is strongly connected, and each node \(i \in \mathcal{V}\) has a self-loop, i.e., the edge \((i,i)\) exists.
\end{ass}

Assumption~\ref{ass:strong-connectivity} can be relaxed to the setting of a periodically strongly connected graph sequence. Specifically, if there exists a positive integer \(C \geq 1\) such that for any \(t \geq 0\), the union of edge sets
\(
\mathcal{E}_t^C := \bigcup_{i = tC}^{(t+1)C - 1} \mathcal{E}_i
\)
forms a strongly connected graph over \(C\) consecutive iterations, then the sequence is said to be \(C\)-strongly connected.

Each agent \(i\) independently determines the values of \([A_t]_{ij}\) for its in-neighbors \(j \in \mathcal{N}_{i,t}^{\mathrm{in}}\), while the corresponding values of \([B_t]_{ij}\) are determined by its out-neighbors. We further impose the following assumptions on the matrices \(A_t\) and \(B_t\).

\begin{ass}
\label{ass:at}
For any $t \ge 0$, \(A_t\) is row-stochastic associated with \(\mathcal{G}_t\), i.e., \(A_t \mathbf{1} = \mathbf{1}\), and for some constant \(a > 0\), it satisfies
\[
\min\nolimits^{+}(A_t) \geq a, \quad \forall t \geq 0,
\]
where \(\min\nolimits^{+}(A_t)\) denotes the smallest positive entry in \(A_t\).
\end{ass}

\begin{ass}
\label{ass:bt}
For any $t \ge 0$, \(B_t\) is column-stochastic associated with \(\mathcal{G}_t\), i.e., \(\mathbf{1}^\top B_t = \mathbf{1}^\top\), and for some constant \(b > 0\), it satisfies
\[
\min\nolimits^{+}(B_t) \geq b, \quad \forall t \geq 0,
\]
where \(\min\nolimits^{+}(B_t)\) denotes the smallest positive entry in \(B_t\).
\end{ass}

\section{CONVERGENCE ANALYSIS}\label{section-4}
This section presents a theoretical convergence analysis of the proposed TV-HSGT algorithm. We first provide several necessary preliminary lemmas in Subsection \ref{sec-4-1}, and then give the main theoretical results in Subsection \ref{sec-4-2}.

\subsection{Preliminary Lemmas}\label{sec-4-1}
Prior to conducting the convergence analysis, this subsection introduces several auxiliary lemmas that lay the theoretical foundation for the subsequent main results.

\begin{lemma}\cite[]{qu2017harnessing}
\label{lemma:mu}
Suppose that $f(x)$ is $\mu$-strongly convex and $L_g$-smooth. Then, for any $x \in \mathbb{R}^d$, if the step size satisfies $0 < \alpha < \frac{2}{ \mu+L_g}$, the following inequality holds
\begin{align}
\label{eq:qu-gradient-step-contraction}
\left\| x - \alpha \nabla f(x) - x^* \right\| \leq (1 - \mu \alpha) \left\| x - x^* \right\|,
\end{align}
where $x^*$ denotes the optimal solution to $f(x)$.
\end{lemma}

\begin{lemma}\cite[]{liao2022compressed}
\label{lemmazeta}
For any integer $k \geq 1$ and any set of vectors $\mathbf{m}_i \in \mathbb{R}^{n \times d}$, it holds that
\begin{align}
\label{eq:sum-square-bound}
\left\| \sum_{i=1}^{k} \mathbf{m}_i \right\|^2 \leq k \sum_{i=1}^{k} \| \mathbf{m}_i \|^2.
\end{align}
Moreover, for any constant $\zeta > 1$, we have
\begin{align}
\label{eq:zeta-bound}
\left\| \sum_{i=1}^{k} \mathbf{m}_i \right\|^2 \leq \zeta \| \mathbf{m}_1 \|^2 + \frac{(k-1)\zeta}{\zeta -1} \sum_{i=2}^{k} \| \mathbf{m}_i \|^2.
\end{align}
\end{lemma}

\begin{lemma}\cite[]{nguyen2023distributed}
\label{lem:average-gradient-diff}
Suppose that $f_{i,t}$ is $L_g$-smooth. Then, the following inequality holds
\begin{align}
\label{eq:weighted-gradient-bound}
\left\| h_t(\mathbf{x}_t) - \nabla f_t(\hat{x}_t) \right\| \leq \frac{L_g}{\sqrt{n}} \left\| \mathbf{x}_t - \hat{\mathbf{x}}_t \right\|,
\end{align}
where $h_t(\mathbf{x}_t) := \frac{1}{n} \sum_{i=1}^n \nabla f_{i,t}(x_{i,t})$, $\hat{x}_t := \sum_{i=1}^n [\phi_t]_i x_{i,t}$, $\mathbf{x}_t = [x_{1,t},\ x_{2,t},\ \dots,\ x_{n,t}]^\top \in \mathbb{R}^{n\times d}$, $\mathbf{x}_t = \mathbf{1}_n \otimes \hat{x}_t^\top$ and $\phi_t$ is a stochastic vector.
\end{lemma}

\begin{lemma}\cite[]{nguyen2022distributed}
\label{lem:weighted-variance}
Give a set of vectors $\{u_i \}_{i \in \mathcal{V}}\subset \mathbb{R}^d$ and nonnegative weights $\{\gamma_i\}_{i \in \mathcal{V}}\subset\mathbb{R}$ satisfying $\sum_{i=1}^n \gamma_i = 1$. Then, for any $\nu \in \mathbb{R}^d$, the following identity holds
\begin{align*}
\| \sum_{i=1}^n \gamma_i u_i - \nu \|^2 
= \sum_{i=1}^n \gamma_i \|u_i - \nu\|^2 - \sum_{i=1}^n \gamma_i \| u_i - \sum_{j=1}^n \gamma_j u_j \|^2.
\end{align*}
\end{lemma}

\begin{lemma}\cite[]{nguyen2022distributed}
\label{lem:left-eigenvector}
Under Assumptions~\ref{ass:strong-connectivity} and~\ref{ass:at}, there exists a corresponding sequence of stochastic vectors $\{\phi_t\}$ such that 
\begin{align}
\label{definephi}
\phi_{t+1}^\top A_t = \phi_t^\top, \forall t\geq 0.
\end{align}
Moreover, for all $i \in \mathcal{V}$ and $t \geq 0$, it holds that
\(
[\phi_t]_i \geq \frac{a^n}{n}.
\)
\end{lemma}

\begin{lemma}\cite[]{nedic2023ab}
\label{lem:right-eigenvector}
Let Assumptions~\ref{ass:strong-connectivity} and \ref{ass:bt} hold. Define the vector sequence $\pi_t$ by
\begin{align}
\label{definepi}
\pi_{t+1} = B_t \pi_t,  \quad \text{with initial value } \pi_0 = \mathbf{1}/n.
\end{align}
Then, for any $t \geq 0$, $\pi_t$ is a stochastic vector satisfying
\(
[\pi_t]_i \geq \frac{b^n}{n}, \forall i\in \mathcal{V}.
\)
\end{lemma}

If the graph sequence $\{\mathcal{G}_t\}$ satisfies the strong connectivity condition over a period of length $C > 1$, then the results of Lemmas~\ref{lem:left-eigenvector} and \ref{lem:right-eigenvector} can be extended. Specifically, for all $t \geq 0$, there exist stochastic vector sequences $\{\phi_t\}$ and $\{\pi_t\}$ such that the following equalities hold~\cite[]{nguyen2022distributed,nedic2023ab,10337617}
\begin{align*}
\phi_{t+C}^\top \left(A_{t+C-1} \cdots A_{t+1} A_t\right) = \phi_t^\top, \\
\pi_{t+C} = \left(B_{t+C-1} \cdots B_{t+1} B_t\right) \pi_t.
\end{align*}
Moreover, for all $i \in \mathcal{V}$, these vector sequences satisfy the following lower bounds
\(
[\phi_t]_i \geq \frac{a^{nC}}{n}, \quad [\pi_t]_i \geq \frac{b^{nC}}{n}.
\)

Let $\mathcal{G} = (\mathcal{V}, \mathcal{E})$ be a strongly connected directed graph, and let the weight matrices $A$ and $B$ be consistent with the structure of $\mathcal{G}$. Denote by $\mathrm{D}(\mathcal{G})$ the diameter of the graph and by $\mathrm{K}(\mathcal{G})$ its maximal edge utility \cite[]{nedic2023ab}. The following lemmas describe the contraction properties satisfied by the matrices $A$ and $B$.

\begin{lemma}\cite[]{nguyen2022distributed}
\label{le:phi}
Let $A$ be a row-stochastic matrix, $\phi$ be a stochastic vector, and $\pi$ be a nonnegative vector such that $\pi^\top A = \phi^\top$. For a set of vectors $\{ x_i \in \mathbb{R}^d\}_{i=1}^n$, define $\hat{x}_\phi = \sum_{i=1}^n \phi_i x_i$. Then, it holds that
\begin{align*}
    \sqrt{\sum_{i=1}^n \pi_i \left\| \sum_{j=1}^n A_{ij} x_j - \hat{x}_\phi \right\|^2} \leq c \sqrt{\sum_{j=1}^n \phi_j \left\| x_j - \hat{x}_\phi \right\|^2},
\end{align*}
where the scalar $c \in (0,1)$ is defined by
\begin{align*}
c = \sqrt{1 - \frac{\min(\pi) \cdot (\min^+(A))^2}{\max^2(\phi) \cdot \mathrm{D}(\mathcal{G}) \cdot \mathrm{K}(\mathcal{G})}}.
\end{align*}
\end{lemma}

\begin{lemma}\cite[]{nedic2023ab}
\label{le:pi}
Let $B$ be a column-stochastic matrix, and let $\nu$ be a stochastic vector with strictly positive elements, i.e., $\nu_i > 0$ for all $i \in \mathcal{V}$. Let $\pi = B \nu$. Then, for any set of vectors $\{ y_i \in \mathbb{R}^d\}_{i=1}^n$, it holds that
\begin{align*}
    \sqrt{\sum_{i=1}^n \pi_i \left\| \frac{1}{\pi_i} \sum_{j=1}^n B_{ij} y_j -  \sum_{j=1}^n y_j \right\|^2} \leq \tau \sqrt{\sum_{i=1}^n \nu_i \left\| \frac{y_i}{\nu_i} - \sum_{j=1}^n y_j \right\|^2},
\end{align*}
where the scalar $\tau \in (0,1)$ is given by
\begin{align*}
\tau = \sqrt{1 - \frac{\min^2(\nu) \cdot (\min^+(B))^2}{\max^2(\nu) \cdot \max(\pi) \cdot \mathrm{D}(\mathcal{G}) \cdot \mathrm{K}(\mathcal{G})}}.
\end{align*}
\end{lemma}

\subsection{Main Results}\label{sec-4-2}

This subsection establishes the key theoretical results on the convergence of the proposed algorithm. To simplify the mathematical exposition, we uniformly use the notation $\mathbb{E}[\cdot]$ to denote the expectation operator throughout the subsequent proofs and derivations. Unless otherwise specified, all expectations are interpreted as conditional expectations with respect to the filtration $\mathcal{F}_t$, that is, we adopt the convention $\mathbb{E}[\cdot] := \mathbb{E}[\cdot \mid \mathcal{F}_t]$. The analysis focuses on bounding four critical error terms in terms of conditional expectations, which are the optimality error $\mathbb{E}[\|\hat{x}_t - x^*_t\|^2]$, the consensus error $\mathbb{E}[\|\mathbf{x}_t - \hat{\mathbf{x}}_t \|_{\phi_t}^2]$, the gradient tracking error $\mathbb{E}[S^2(\mathbf{y}_t, \pi_t)]$, and the hybrid stochastic gradient estimation error $\mathbb{E} \left[ \| \mathbf{z}_{t+1} - \nabla F_{t+1}(\mathbf{x}_{t+1}) \|^2 \right]$. Here, the consensus error is measured by the weighted norm $\|\mathbf{x}_t - \hat{\mathbf{x}}_t \|_{\phi_t}$, and the gradient tracking deviation is quantified by $S(\mathbf{y}_t, \pi_t)$, which are defined as follows 
\begin{align}
\|\mathbf{x}_t - \hat{\mathbf{x}}_t \|_{\phi_t}
&= \sqrt{ \sum_{i=1}^n [\phi_t]_i \|x_{i,t} - \hat{x}_t\|^2 }, \label{eq:gongshiwucha} \\
S(\mathbf{y}_t, \pi_t) 
&= \sqrt{ \sum_{i=1}^n [\pi_t]_i \left\| \frac{y_{i,t}}{[\pi_t]_i} -  \sum_{j=1}^n y_{j,t} \right\|^2 }, \label{eq:grad-track-error}
\end{align}
where $\hat{x}_t := \sum_{i=1}^n [\phi_t]_i x_{i,t}$ represents the weighted average of local decision variables. The stochastic weight sequences $\{\phi_t\}$ and $\{\pi_t\}$ are defined by equations~\eqref{definephi} and~\eqref{definepi}, respectively. Moreover, $x^*_t$ denotes the optimal solution to problem~\eqref{problem} at time $t$. In the later analysis, we denote $\mathbf{x}_t = [x_{1,t},\ x_{2,t},\ \dots,\ x_{n,t}]^\top \in \mathbb{R}^{n\times d}$ (same to $\mathbf{y}_t$ and $\mathbf{z}_t$), $\mathbf{\hat{x}}_t = \mathbf{1}_n \otimes \hat{x}_t^\top$, $\mathbf{x}_t^* = \mathbf{1}_n \otimes (x_t^*)^\top$, $\nabla F_{t}(\mathbf{x}_{t})=[\nabla f_{1,t}(x_{i,t}),\ \nabla f_{2,t}(x_{2,t}),\ \dots,\ \nabla f_{n,t}(x_{n,t})]^\top$, and $h_t(\mathbf{x}_t) := \frac{1}{n} \sum_{i=1}^n \nabla f_{i,t}(x_{i,t})$. 

To facilitate the convergence analysis of the proposed algorithm under time-varying directed topologies, we introduce a set of auxiliary parameters: $\kappa_t \geq 1$, $\varphi_t \geq 1$, $\gamma_t \in (0,1]$, $\psi_t > 0$, $\tau_t \in (0,1)$, $c_t \in (0,1)$, $\nu_t > 0$, and $\zeta_t > 0$. These quantities are defined as follows
\begin{align}
\label{definecon}
\varphi_t &= \sqrt{\frac{1}{\min(\phi_t)}}, \, 
\kappa_t = \sqrt{\frac{1}{\min(\pi_t)}}, \, 
\gamma_t = \sqrt{ \max_{i \in \mathcal{V}} \left( [\phi_{t}]_i [\pi_t]_i \right)},\notag \\
\psi_t &= \kappa_t^2, \,
c_t = \sqrt{1 - \frac{\min(\phi_{t+1}) \, a^2}{\max^2(\phi_t) \, \mathrm{D}(\mathcal{G}_t)\mathrm{K}(\mathcal{G}_t)}}, \notag \\
\nu_t &= \frac{6L_g^2 (c \varphi_{t+1} + 1)^2 \gamma_t^2\tau^2 \psi_t}{1 - \tau} , \,
\zeta_t = \frac{6 L_g^2 (c \varphi_{t+1} + \varphi_t)^2 \tau^2 \psi_t}{1 - \tau},\notag \\
\tau_t &= \sqrt{1 - \frac{\min^2(\pi_t) \, b^2}{\max^2(\pi_t) \max(\pi_{t+1}) \mathrm{D}(\mathcal{G}_t)\mathrm{K}(\mathcal{G}_t)}}, 
\end{align} 
where $c \in (0,1)$ and $\tau \in (0,1)$ are constant upper bounds for the time-varying quantities $c_t$ and $\tau_t$, respectively. Additionally, let $\eta$ denote a uniform lower bound of the inner product $\phi_t^\top \pi_t$. Since $\phi_t$ and $\pi_t$ are stochastic vectors, it follows that $\phi_t^\top \pi_t \leq 1$, and hence $\eta \leq 1$. For notational conciseness and in order to establish uniform bounds on the algorithm’s performance, we also introduce constant upper bounds $\psi > 0$, $\kappa > 1$, and $\varphi > 1$ for $\psi_t$, $\kappa_t$, and $\varphi_t$, respectively. The bounding conditions are then given by
\begin{align}
\label{definec}
\max_{t \geq 0} c_t \leq c, \quad \max_{t \geq 0} \tau_t \leq \tau, \quad \min_{t \geq 0} \phi_{t}^\top \pi_t \geq \eta, \notag \\
\max_{t \geq 0} \psi_t \leq \psi, \quad 
\max_{t \geq 0} \kappa_t \leq \kappa, \quad 
\max_{t \geq 0} \varphi_t \leq \varphi.
\end{align}

In the following, we present Lemmas \ref{lemmayit} to \ref{zt1xt1}, which establish bounds on several key terms used in the subsequent convergence analysis. Detailed proofs can be found in the appendix.
\begin{lemma}
\label{lemmayit}
Under Assumptions~\ref{ass:Lg} and \ref{ass:bt}, the following inequality holds for all \( t \geq 0\)
\vspace{-0.8em}  
\begin{small} 
\begin{align}
 \mathbb{E} \left[ \left \| \sum_{i=1}^n y_{i,t} \right\|^2 \right] &\leq 2n \mathbb{E} \left[\left \|  \mathbf{z}_t  -  \nabla F_t(\mathbf{x}_{t})\right\|^2\right] + 2 L_g^2 n \varphi_t^2 \mathbb{E} \left[ \| \hat{x}_t - x_t^* \|^2  \right] \notag \\
 & \quad+ 2 L_g^2 n \varphi_t^2 \mathbb{E} \left[\| \mathbf{x}_t - \hat{\mathbf{x}}_t \|_{\phi_t}^2\right].
\end{align}
\end{small}
\end{lemma}

\begin{lemma}
\label{le:ypai}
Under Assumptions~\ref{ass:strong-connectivity} and \ref{ass:bt}, the following inequality holds for all \( t \geq 0\)
\vspace{-0.8em}  
\begin{small}
\begin{equation}
\label{ypai}
\begin{aligned}
\mathbb{E} \left[ \| \mathbf{y}_t \|_{\pi_t^{-1}}^2 \right] & \leq 2n \mathbb{E} \left[\left \|  \mathbf{z}_t  -  \nabla F_t(\mathbf{x}_{t})\right\|^2\right] + 2 L_g^2 n \varphi_t^2 \mathbb{E} \left[ \| \hat{x}_t - x_t^* \|^2  \right]\\
& \quad + 2 L_g^2 n \varphi_t^2 \mathbb{E} \left[\| \mathbf{x}_t - \hat{\mathbf{x}}_t \|_{\phi_t}^2\right]+\mathbb{E} \left[ S^2(\mathbf{y}_t, \pi_t) \right].
\end{aligned}
\end{equation}
\end{small}
\end{lemma}

\begin{lemma}
\label{le:hattxing}
Under Assumptions~\ref{ass:strongtu}, \ref{ass:Lg}, \ref{ass:sigma}, and \ref{ass:strong-connectivity}, if \(0 < \alpha < \frac{2}{n(\mu +L_g)\phi_{t}^\top \pi_t}\), it holds that for all \( t \geq 0\)
\begin{align}
   & \mathbb{E} \left[ \| \hat{x}_{t+1} - x_{t+1}^* \|^2  \right]  \notag \\ 
   & \leq (1 - \mu \alpha n \phi_{t}^\top \pi_t) \mathbb{E} \left[ \| \hat{x}_t - x_t^* \|^2  \right]  + \frac{4 \alpha}{\mu n \phi_{t}^\top \pi_t} \mathbb{E} \left[ S^2(\mathbf{y}_t, \pi_t) \right] \notag \\
   & \quad+ \frac{4 \alpha (\phi_{t}^\top \pi_t)}{\mu} \mathbb{E} \left[\left \|  \mathbf{z}_t  -  \nabla F_t(\mathbf{x}_{t})\right\|^2\right]  + \frac{4}{\mu \alpha n \phi_{t}^\top \pi_t} \| x_t^* - x_{t+1}^* \|^2 \notag \\
   & \quad + \frac{4 \alpha (\phi_{t}^\top \pi_t) L_g^2 \varphi_t^2}{\mu} \mathbb{E} \left[\| \mathbf{x}_t - \hat{\mathbf{x}}_t \|_{\phi_t}^2\right].
\end{align}
\end{lemma}

\begin{lemma} 
\label{xtxhat}
Under Assumptions~\ref{ass:Lg}, \ref{ass:sigma}, and \ref{ass:strong-connectivity}, the following inequality holds for all \( t \geq 0\)
\begin{align} 
\label{1xalpha}
&\mathbb{E} \left[ \| \mathbf{x}_{t+1} - \hat{\mathbf{x}}_{t+1} \|_{\phi_{t+1}}^2  \right] \notag \\ 
& \leq \left( \frac{1 + c^2}{2} + \frac{2 \alpha^2 c^2 \gamma_t^2 (1 + c^2) L_g^2 n \varphi_t^2}{1 - c^2} \right) \mathbb{E} \left[ \|\mathbf{x}_t - \hat{\mathbf{x}}_t\|_{\phi_t}^2 \right]   \notag \\
& \quad + \frac{\alpha^2 c^2 \gamma_t^2 (1 + c^2)}{1 - c^2} \mathbb{E} \left[ S^2(\mathbf{y}_t, \pi_t) \right] \notag \\
& \quad + \frac{2 \alpha^2 c^2 \gamma_t^2 (1 + c^2) L_g^2 n \varphi_t^2}{1 - c^2} \mathbb{E} \left[ \| \hat{x}_t - x_t^* \|^2 \right]  \notag \\
& \quad+ \frac{2 \alpha^2 c^2 \gamma_t^2 (1 + c^2) n}{1 - c^2} \mathbb{E} \left[ \| \mathbf{z}_t - \nabla F_t(\mathbf{x}_t) \|^2 \right].
\end{align}
\end{lemma}

\begin{lemma} 
\label{xt1xt}
Under Assumptions~\ref{ass:strong-connectivity} and \ref{ass:at}, the following inequality holds for all \( t \geq 0 \)
\vspace{-0.8em}  
\begin{small}
\begin{align*}
& \mathbb{E} \left[ \|\mathbf{x}_{t+1} - \mathbf{x}_t\|^2 \right]\notag \\
&\le \left( 2(c\varphi_{t+1}+\varphi_{t})^2 + 4 \alpha^2 \gamma_t^2 L_g^2 n \varphi_t^2 (c \varphi_{t+1} + 1)^2 \right)
\mathbb{E} \left[ \|\mathbf{x}_t - \hat{\mathbf{x}}_t\|_{\phi_t}^2 \right] \notag \\
&\quad + 4 \alpha^2 \gamma_t^2 (c \varphi_{t+1} + 1)^2 L_g^2 n \varphi_t^2 \mathbb{E} \left[ \| \hat{x}_t - x_t^* \|^2  \right] \notag \\
&\quad + 4 \alpha^2 \gamma_t^2 (c \varphi_{t+1} + 1)^2 n \mathbb{E} \left[\left \|  \mathbf{z}_t  -  \nabla F_t(\mathbf{x}_{t})\right\|^2\right] \notag \\
&\quad + 2 \alpha^2 \gamma_t^2 (c \varphi_{t+1} + 1)^2 \mathbb{E} \left[ S^2(\mathbf{y}_t, \pi_t) \right],
\end{align*}
\end{small}
where \(\varphi_t = \sqrt{\frac{1}{\min(\phi_t)}}\), and \(\gamma_t = \sqrt{\max_{i \in \mathcal{V}} \left( [\phi_{t}]_i [\pi_t]_i \right)}\).
\end{lemma}

\begin{lemma} 
\label{zt1zt}
Under Assumptions~\ref{ass:Lg} and \ref{ass:sigma}, the following inequality holds for all \( t \geq 0 \)
\vspace{-0.8em}  
\begin{scriptsize} 
\begin{align*}
& \mathbb{E} \left[ \| \mathbf{z}_{t+1} - \mathbf{z}_{t} \|^2  \right] \notag \\
&\leq \left[ 6 L_g^2 (c \varphi_{t+1} + \varphi_t)^2 + 12 \alpha^2 L_g^4 n \varphi_t^2 (c \varphi_{t+1} + 1)^2 \gamma_t^2 \right] \mathbb{E} \left[ \| \mathbf{x}_t - \hat{\mathbf{x}}_t \|_{\phi_t}^2 \right] \notag \\
&\quad + 12 \alpha^2 L_g^4 n \varphi_t^2 (c \varphi_{t+1} + 1)^2 \gamma_t^2 \mathbb{E} \left[ \| \hat{x}_t - x_t^* \|^2  \right] \notag \\
&\quad + \left[ 12 \alpha^2 L_g^2 n (c \varphi_{t+1} + 1)^2 \gamma_t^2 + 3 \beta^2 \right] \mathbb{E} \left[ \| \mathbf{z}_t - \nabla F_t(\mathbf{x}_t) \|^2 \right] \notag \\
&\quad + 6 \alpha^2 L_g^2 (c \varphi_{t+1} + 1)^2 \gamma_t^2 \mathbb{E} \left[ S^2(\mathbf{y}_t, \pi_t) \right] + 6 \beta^2 n q_t^2 + 6 \beta^2 n \sigma^2.
\end{align*}
\end{scriptsize}
\end{lemma}

\begin{lemma} 
\label{le:yt1pait1}
Under Assumptions~\ref{ass:Lg}, \ref{ass:sigma}, and \ref{ass:strong-connectivity}, it holds that for all \( t \geq 0\)
\begin{align}\label{S-bound}
& \mathbb{E} \left[ S^2(\mathbf{y}_{t+1}, \pi_{t+1})  \right] \notag \\
& \leq \tau \mathbb{E} \left[ S^2(\mathbf{y}_t, \pi_t) \right] + \frac{\tau^2 \kappa_t^2}{1-\tau}\mathbb{E} \left[ \|\mathbf{z}_{t+1} - \mathbf{z}_t\|^2\right].
\end{align}
\end{lemma}


\begin{lemma} 
\label{zt1xt1}
Under Assumptions~\ref{ass:Lg} and \ref{ass:sigma}, it holds that for all \( t \geq 0\) and $\zeta_0>0$
\vspace{-0.4cm}
\begin{small} 
\begin{align}
& \mathbb{E} \left[ \| \mathbf{z}_{t+1} - \nabla F_{t+1}(\mathbf{x}_{t+1})  \|^2 \right] \notag \\
& \leq (1 - \beta)^2(1+\zeta_0) \mathbb{E} \left[ \| \mathbf{z}_t - \nabla F_t(\mathbf{x}_t) \|^2 \right] 
+ (8+\zeta_0^{-1})n (1 - \beta)^2 q_t^2 \notag \\
& \quad + n \beta^2 \sigma^2 + 12 (1 - \beta)^2 L_g^2 \mathbb{E} \left[ \| \mathbf{x}_{t+1} - \mathbf{x}_t \|^2 \right],\label{z_bound}
\end{align}
\end{small} 
where \(q_t\) is defined in \eqref{qt}.
\end{lemma}

To facilitate the analysis, we establish a coupled relationship among the expectations of the following four error terms by defining the vector \(V_t\) as
\begin{align}
V_t = \begin{bmatrix}
\mathbb{E} \left[ \|\mathbf{x}_t - \hat{\mathbf{x}}_t\|_{\phi_t}^2 \right]  \\
\mathbb{E} \left[ S^2(\mathbf{y}_t, \pi_t) \right] \\
\mathbb{E} \left[ \| \hat{x}_t - x_t^* \|^2 \right] \\
\mathbb{E} \left[ \| \mathbf{z}_{t} - \nabla F_{t}(\mathbf{x}_{t}) \|^2 \right]
\end{bmatrix}.
\end{align}

Based on the results of the previously established lemmas, the following linear inequality system can be established.

\begin{proposition}
\label{th0}
Let the collections of sequences~\mbox{$\{\{ x_{i,t} \}_{i=1}^n \}_{t=1}^T$}, 
\mbox{$\{\{ z_{i,t} \}_{i=1}^n \}_{t=1}^T$}, and 
\mbox{$\{\{ y_{i,t} \}_{i=1}^n \}_{t=1}^T$} be generated by  Algorithm \ref{alg:reactmomentum}. 
Under Assumptions~\ref{ass:strongtu}–\ref{ass:bt}, the following linear inequality system holds
\begin{align}
\label{vt2}
V_{t+1} \leq M(\alpha)V_t + b_{1,t} + b_{2},
\end{align}
where \(b_{1,t}\) and \(b_{2}\) are vectors given by
\begin{align} 
b_{1,t} &= \begin{bmatrix}
0, k_1 q_t^2,
k_2 p_t^2,
k_3 q_t^2
\end{bmatrix}^\top, \label{bt1}\\
b_{2} &= \begin{bmatrix}
0,\frac{6 n\tau^2 \psi}{1 - \tau}\beta^2  \sigma^2,
0,
2n \beta^2 \sigma^2
\end{bmatrix}^\top. \label{bt2}
\end{align}
The coefficient parameters are defined as \(k_1 = \frac{6 n \beta^2 \tau^2 \psi}{1 - \tau}\), \(k_2 = \frac{4}{\mu \alpha n \eta}\), and \(k_3 = (8+\zeta_0^{-1})n (1 - \beta)^2\) with $\zeta_0 \in(0, \frac{1}{(1-\beta)^2}-1)$.
\end{proposition}
\begin{proof}
By applying Lemma \ref{zt1zt} to \eqref{S-bound}, we get the following inequality
\begin{align}
& \mathbb{E} \left[ S^2(\mathbf{y}_{t+1}, \pi_{t+1})  \right] \notag \\
& \leq \left[ \tau + \frac{\tau^2}{1 - \tau}  \kappa_t^2 \cdot 6 \alpha^2 L_g^2 
(c \varphi_{t+1} + 1)^2 \gamma_t^2 \right] 
\mathbb{E} \left[ S^2(\mathbf{y}_t, \pi_t) \right] \notag \\
& \quad + \frac{\tau^2}{1 - \tau}  \kappa_t^2 \Big[ 
6 L_g^2 (c \varphi_{t+1} + \varphi_t)^2 \notag \\
& \qquad + 12 \alpha^2 L_g^4 n \varphi_t^2 (c \varphi_{t+1} + 1)^2 \gamma_t^2 
\Big] \mathbb{E} \left[ \| \mathbf{x}_t - \hat{\mathbf{x}}_t \|_{\phi_t}^2 \right] \notag \\
& \quad + \frac{\tau^2}{1 - \tau}  \kappa_t^2 \cdot 
12 \alpha^2 L_g^4 n \varphi_t^2 (c \varphi_{t+1} + 1)^2 \gamma_t^2 
\mathbb{E} \left[ \| \hat{x}_t - x_t^* \|^2  \right] \notag \\
& \quad + \frac{\tau^2}{1 - \tau}  \kappa_t^2 \Big[  12 \alpha^2 L_g^2 n (c \varphi_{t+1} + 1)^2 \gamma_t^2 \notag \\
& \qquad + 3 \beta^2 \Big] 
\mathbb{E} \left[ \| \mathbf{z}_t - \nabla F_t(\mathbf{x}_t) \|^2 \right] \notag \\
& \quad + \frac{\tau^2}{1 - \tau}  \kappa_t^2 
\left[ 6 \beta^2 n q_t^2 + 6 \beta^2 n \sigma^2 \right]. 
\label{yt1pait1}
\end{align}
By substituting the result of Lemma~\ref{xt1xt}, which bounds \(\mathbb{E} [ \| \mathbf{x}_{t+1} - \mathbf{x}_t \|^2 ]\), into \eqref{z_bound} gives
\begin{align}
& \mathbb{E} \left[ \| \mathbf{z}_{t+1} - \nabla F_{t+1}(\mathbf{x}_{t+1}) \|^2 \right]  \notag \\
& \leq 
(1 - \beta)^2 \left[ (1 + \zeta_0) + 48 \alpha^2 \gamma_t^2 L_g^2 n (c \varphi_{t+1} + 1)^2 \right]\cdot 
      \notag \\
& \quad \mathbb{E} \left[ \| \mathbf{z}_t - \nabla F_t(\mathbf{x}_t) \|^2 \right]+ 12 (1 - \beta)^2 L_g^2 \Big[ 
        2 (c \varphi_{t+1} + \varphi_t)^2 \notag \\
& \quad + 4 \alpha^2 \gamma_t^2 L_g^2 n \varphi_t^2 (c \varphi_{t+1} + 1)^2 
     \Big] \mathbb{E} \left[ \| \mathbf{x}_t - \hat{\mathbf{x}}_t \|_{\phi_t}^2 \right] \notag \\
& \quad + 48 (1 - \beta)^2 \alpha^2 \gamma_t^2 L_g^4 n \varphi_t^2 (c \varphi_{t+1} + 1)^2 
     \mathbb{E} \left[ \| \hat{x}_t - x_t^* \|^2 \right] \notag \\
& \quad + 24 (1 - \beta)^2 \alpha^2 \gamma_t^2 L_g^2 (c \varphi_{t+1} + 1)^2 
     \mathbb{E} \left[ S^2(\mathbf{y}_t, \pi_t) \right] \notag \\
& \quad + (8+\zeta_0^{-1})n (1 - \beta)^2 q_t^2  + n \beta^2 \sigma^2
\label{zt1ft1}.
\end{align}
Then, combined with Lemmas~\ref{le:hattxing} and \ref{xtxhat}, it follows that under the step size condition \(0 < \alpha < \frac{2}{n (\mu + \eta) L_g}\), the vector \(V_t\) satisfies the following dynamical system
\begin{align}
V_{t+1} \leq M_t(\alpha)V_t + B_t^{'},
\end{align}
where \(M_t(\alpha)\) can be expressed as
\[
\begin{bmatrix}
\frac{1 + c^2}{2} + \alpha^2 m_t^1 & \alpha^2 m_t^2 & \alpha^2 m_t^1 & \alpha^2 m_t^3 \\
m_t^4 + \alpha^2 m_t^6 & \tau + \alpha^2 m_t^5 & \alpha^2 m_t^6 & m_t^7 + \alpha^2 m_t^8 \\
\alpha m_t^9 & \alpha m_t^{10} & 1 - \alpha m_t^{11} & \alpha m_t^{12} \\
m_t^{14} + \alpha^2 m_t^{15} & \alpha^2 m_t^{13} & \alpha^2 m_t^{15} & m_0 + \alpha^2 m_t^{16}
\end{bmatrix},
\]
and \(B_t^{'} = b'_{1,t} + b'_{2,t}\) with
\[
b'_{1,t} = \begin{bmatrix}
0,\frac{\tau^2 \psi_t}{1 - \tau} 6 \beta^2 n  q_t^2,
\frac{4}{\mu \alpha n \phi_{t}^\top \pi_t} p_t^2,
(8+\zeta_0^{-1})n (1 - \beta)^2 q_t^2
\end{bmatrix}^\top,\]
\[
b'_{2,t} = \begin{bmatrix}
0,\frac{\tau^2 \psi_t}{1 - \tau}6 \beta^2 n \sigma^2,
0,
n \beta^2 \sigma^2
\end{bmatrix}^\top.
\]

\noindent By introducing the parameter definitions in~\eqref{definecon}, the entries in \(M_t(\alpha)\) are defined as follows
\begin{align}
m_t^1 &= \frac{2nL_g^2 \varphi_t^2 c^2 (1 + c^2) \gamma_t^2}{1 - c^2}, \quad 
m_t^2 = \frac{(1 + c^2) c^2 \gamma_t^2}{1 - c^2}, \notag \\
m_t^3 &= \frac{2(1 + c^2) c^2 \gamma_t^2 n}{1 - c^2}, \quad
m_t^4 = \zeta_t, \quad 
m_t^5 = \nu_t, \notag \\ 
m_t^6 &= 2 n L_g^2 \nu_t, \quad 
m_t^7 = \frac{3 \psi_t \beta^2\tau^2}{1 - \tau}, \quad 
m_t^8 = 2 n \nu_t, \notag \\
m_t^9 &= \frac{4 (\phi_{t}^\top \pi_t) L_g^2 \varphi_t^2}{\mu}, \quad
m_t^{10} = \frac{4}{\mu n \phi_{t}^\top \pi_t}, \quad
m_t^{11} = \mu n \phi_{t}^\top \pi_t, \notag \\
m_t^{12} & = \frac{4 (\phi_{t}^\top \pi_t)}{\mu}, \quad
m_t^{13}= 24 (1 - \beta)^2 \gamma_t^2 L_g^2 (c \varphi_{t+1} + 1)^2, \notag \\ 
m_t^{14} &= 24 (1 - \beta)^2 L_g^2 (c \varphi_{t+1} + \varphi_{t})^2, \quad
m_t^{15} = 2n L_g^2 \varphi_t^2 m_t^{13},  \notag \\
m_t^{16} &= 2n m_t^{13}, m_0=(1 - \beta)^2(1 + \zeta_0).
\end{align}

\noindent By substituting the upper and lower bounds of parameters defined in~\eqref{definec}, the upper bound of \(M_t(\alpha)\) can be given by
\begin{align}
\label{ma}
M(\alpha)= \begin{bmatrix}
\frac{1 + c^2}{2} + \alpha^2 m_1 & \alpha^2 m_2 & \alpha^2 m_1 & \alpha^2 m_3 \\
m_4 + \alpha^2 m_6 & \tau + \alpha^2 m_5 & \alpha^2 m_6 & m_7 + \alpha^2 m_8 \\
\alpha m_9 & \alpha m_{10} & 1 - \alpha m_{11} & \alpha m_{12} \\
m_{14} + \alpha^2 m_{15} & \alpha^2 m_{13} & \alpha^2 m_{15} & m_0 + \alpha^2 m_{16}
\end{bmatrix},
\end{align}
satisfying \(M_t(\alpha) \leq M(\alpha)\), where the time-varying coefficients can be upper bounded by the following constants
\begin{align}
m_1 &= \frac{2nL_g^2 \varphi^2 c^2 (1 + c^2)}{1 - c^2}, \quad 
m_2 = \frac{(1 + c^2) c^2}{1 - c^2}, \notag \\
m_3 &= \frac{2(1 + c^2) c^2 n}{1 - c^2}, \quad 
m_4 = \zeta, \quad 
m_5 = \nu, \quad 
m_6 = 2 n L_g^2 \nu,  \notag \\
m_7 &= \frac{3 \psi \beta^2 \tau^2}{1 - \tau}, \quad 
m_8 = 2 n \nu, \quad 
m_9 = \frac{4 L_g^2 \varphi^2}{\mu},  \label{contant-m} \\
m_{10} &= \frac{4}{\mu n \eta }, \quad
m_{11} = \mu n, \quad
m_{12} = \frac{4}{\mu}, \notag \\
m_{13} &= 24 (1 - \beta)^2 L_g^2 (c \varphi + 1)^2,
m_{14} = 24 (1 - \beta)^2 L_g^2 \varphi^2(1+c)^2, \notag \\
m_{15} &= 2n L_g^2 \varphi^2 m_{13}, \quad
m_{16} = 2n m_{13},m_0=(1 - \beta)^2(1 + \zeta_0).\notag
\end{align}
Here \(\zeta = \frac{24 L_g^2  \varphi^2 \tau^2 \psi}{1 - \tau}\), \(\nu = \frac{6 L_g^2 (c \varphi + 1)^2 \tau^2 \psi}{1 - \tau}\). Consequently, \(B_t^{'}\) can be bounded by $B'=b_{1,t}+b_2$ defined in \eqref{bt1} and \eqref{bt2}
Thus, the proof is completed.
\end{proof}

To obtain the main theoretical result, we establish a regret bound for the proposed TV-HSGT algorithm under time-varying directed networks. The result demonstrates that the algorithm effectively reduces the variance caused by stochastic gradients.

\begin{thm} 
\label{th3}
Let the collections of sequences~\mbox{$\{\{ x_{i,t} \}_{i=1}^n \}_{t=1}^T$}, 
\mbox{$\{\{ z_{i,t} \}_{i=1}^n \}_{t=1}^T$}, and 
\mbox{$\{\{ y_{i,t} \}_{i=1}^n \}_{t=1}^T$} be generated by Algorithm \ref{alg:reactmomentum}. 
Let Assumptions \ref{ass:strongtu}–\ref{ass:bt} hold and the step size $\alpha$ satisfy the condition \eqref{cond-alpha}. Then, there exists a constant \(\tilde{\rho}\in (0,1)\) such that the dynamic regret satisfies
\begin{align*}
R_T^d \leq \mathcal{O} \left( \| V_0 \| + \sum_{t=0}^{T-1} \|  b_{1,t}  \| + \beta^2 \sigma^2T \| b'_2 \| \right),
\end{align*}
where $b_{1,t}$ is defined in \eqref{bt1} and  
\(
b'_{2} = \begin{bmatrix}
0,\frac{6 n\tau^2 \psi}{1 - \tau},
0,
n 
\end{bmatrix}^\top
\).
\end{thm}

\begin{proof}
Recall the linear inequality system \eqref{vt2}, given by \(V_{t+1} \leq M(\alpha)V_t+b_{1,t}+b_2\) for all \(t \geq 0\). The goal is to determine a feasible range for the step size $\alpha$ such that the spectral radius $\rho(\alpha)$ of \(M(\alpha)\) satisfies \(\rho(\alpha) < 1\).
It is sufficient to find a positive vector $\boldsymbol{\delta} = [\delta_1, \delta_2, \delta_3, \delta_4]^\top$ and a range for $\alpha > 0$ such that $M(\alpha) \boldsymbol{\delta} < \boldsymbol{\delta}$ \cite[]{horn2012matrix}. Expanding and  rearranging this inequality element-wisely, we obtain 
\begin{align}
\alpha^2 \left( m_1 \delta_1 + m_2 \delta_2 + m_1 \delta_3 + m_3 \delta_4 \right) &< \frac{1-c^2}{2} \delta_1, \label{eq:01} \\
 \alpha^2 ( m_6 \delta_1 + m_5 \delta_2 + m_6 \delta_3+ m_8 \delta_4 )&< \nonumber\\
 (1-\tau)\delta_2 -& m_4 \delta_1   - m_7 \delta_4, \label{eq:02} \\
 \alpha \left( m_9 \delta_1 + m_{10} \delta_2 + m_{12} \delta_4 \right) &< \alpha m_{11} \delta_3, \label{eq:03}\\
\alpha^2 \left( m_{15} \delta_1 + m_{13} \delta_2 + m_{15} \delta_3 + m_{16} \delta_4 \right)&< \nonumber\\
 (1-m_0)\delta_4 &- m_{14} \delta_1.\label{eq:04}
\end{align}

To ensure these inequalities hold for some $\alpha > 0$, the right-hand sides must be positive, which gives a set of constraints on the components of the vector $\boldsymbol{\delta}$, i.e.,
\begin{align}
    \delta_3 &> \frac{m_9 \delta_1 + m_{10} \delta_2 + m_{12} \delta_4}{m_{11}}, \label{eq:05}\\
    \delta_4 &> \frac{m_{14}}{1-m_0} \delta_1,\label{eq:06}\\
    \delta_2 &> \frac{m_4 \delta_1 + m_7 \delta_4}{1-\tau}. \label{eq:07}
\end{align}

We now construct a feasible positive vector $\boldsymbol{\delta}$ that satisfies the conditions (\ref{eq:05}), (\ref{eq:06}), and (\ref{eq:07}). Let us fix $\delta_1 = 1$.
Based on (\ref{eq:06}), we can set $\delta_4 = \frac{2 m_{14}}{1-m_0}$.
Plugging this into (\ref{eq:07}), we select $\delta_2$ to satisfy
$$
\delta_2 = \frac{2}{1-\tau} \left( m_4 + \frac{2 m_7 m_{14}}{1-m_0} \right).
$$
Finally, based on (\ref{eq:05}), we set $\delta_3$ as
$$
\delta_3 = \frac{2}{m_{11}} \left( m_9 + m_{10}\delta_2 + m_{12}\delta_4 \right).
$$
With this choice, $\boldsymbol{\delta} = [\delta_1, \delta_2, \delta_3, \delta_4]^{\top}$ is a positive vector satisfying the necessary constraints.
Now, substituting these values back into inequalities (\ref{eq:01}), (\ref{eq:02}), and (\ref{eq:04}) to derive upper bounds on $\alpha$ yields
\begin{align*}
\alpha &< \sqrt{ \frac{(1-c^2)\delta_1}{2(m_1 \delta_1 + m_2 \delta_2 + m_1 \delta_3 + m_3 \delta_4)}} := B_1,\\
    \alpha &< \sqrt{\frac{m_4 \delta_1 + m_7 \delta_4}{m_6 \delta_1 + m_5 \delta_2 + m_6 \delta_3 + m_8 \delta_4}} := B_2,\\
\alpha &< \sqrt{\frac{m_{14} \delta_1}{m_{15} \delta_1 + m_{13} \delta_2 + m_{15} \delta_3 + m_{16} \delta_4}} := B_3.
\end{align*}

To summarize, with the constructed positive vector $\mathbf{\delta}$ and the defined constants \eqref{contant-m}, together with Lemma \ref{le:hattxing}, a sufficient condition on the step size $\alpha$ that guarantees $\rho(M(\alpha)) < 1$ is given by
\begin{align}\label{cond-alpha}
    0 < \alpha < \min\{B_1, B_2, B_3,\frac{2}{n(\mu +L_g)\eta}\}.
\end{align}


Recalling that the local function $f_{i,t}$ is $L_g$-smooth and by the definition $f_t (x) := \frac{1}{n} \sum_{i=1}^n  f_{i,t}(x)$, it implies the global function \(f_t(x)\) is also \(L_g\)-smooth, which satisfies
\begin{align}
   f_t(y) \leq f_t(x) + \langle \nabla f_t(x), y - x \rangle + \frac{L_g}{2} \|y - x\|^2. \label{eq:desc}
\end{align}

Let \(y = \hat{x}_t\) and \(x = x_t^*\). Since \(x_t^*\) is the minimizer of \(f_t(x)\), the first-order optimality condition under Assumption \ref{ass:strongtu} implies \(\nabla f_t(x_t^*) = 0\). Substituting these into (\ref{eq:desc}) yields
\[
f_t(\hat{x}_t) \leq f_t(x_t^*) + \langle 0, \hat{x}_t - x_t^* \rangle + \frac{L_g}{2} \|\hat{x}_t - x_t^*\|^2,
\]
which simplifies to
\[
f_t(\hat{x}_t) - f_t(x_t^*) \leq \frac{L_g}{2} \|\hat{x}_t - x_t^*\|^2.
\]
Taking the expectation and summing over \(t\) from 1 to \(T\), we get
\begin{align}
\label{rtchan}
R_T^d \leq \sum_{t=1}^T \left[ \frac{L_g}{2} \mathbb{E} \left[ \| \hat{x}_t - x_t^* \|^2 \right] \right] \leq  \frac{L_g}{2}  \sum_{t=1}^T \| V_t \|.
\end{align}
In any finite-dimensional vector space, all norms are equivalent, so there exist constants \(\lambda_1\) and \(\lambda_2\) satisfying
\begin{align}
\label{lamda12}
\| v \| \leq \lambda_1 \| v \|_\gamma, \quad \| v \|_\gamma \leq \lambda_2 \| v \|.
\end{align}
Substituting \eqref{lamda12} into \eqref{rtchan} gives \(R_T^d \leq \frac{  L_g \lambda_1}{2}\| V_t \|_\gamma\). According to matrix analysis theory~\cite[]{horn2012matrix}, for any \(\gamma > 0\), a matrix norm \(\| \cdot \|_\gamma\) exists such that
\begin{align*}
\| M(\alpha) \|_\gamma \leq \rho(M(\alpha)) + \gamma.
\end{align*}
Letting \(\gamma \in (0, 1 - \rho(M(\alpha)))\) and defining \(
\tilde{\rho} = \rho(M(\alpha)) + \gamma
\), we have \(
\| M(\alpha) \|_\gamma \leq \tilde{\rho} < 1
\). Matrix norm submultiplicativity further implies \(\| N v \|_\gamma \leq \| N\|_\gamma \| v \|_\gamma\) for any matrix \(N\) and vector \(v\). Applying this to the recursion \eqref{vt2}, we obtain
\begin{align*}
    R_T^d \leq \frac{L_g \lambda_1}{2} \sum_{t=1}^T \left[ \tilde{\rho}^t \| V_0 \|_\gamma + \sum_{k=0}^{t-1} \tilde{\rho}^k ( \| b_{1,t} \|_\gamma + \| b_2 \|_\gamma ) \right],
\end{align*}
and applying \eqref{lamda12} again yields
\begin{align*}
    R_T^d &\leq \frac{L_g \lambda_1 \lambda_2}{2} \sum_{t=1}^T \tilde{\rho}^t \| V_0 \| + \frac{L_g \lambda_1 \lambda_2}{2} \sum_{t=1}^T \sum_{k=0}^{t-1} \tilde{\rho}^k \| b_{1,t}  \|  \notag \\
    & \quad + \frac{L_g \lambda_1 \lambda_2}{2} \sum_{t=1}^T \sum_{k=0}^{t-1} \tilde{\rho}^k \| b_2 \|.
\end{align*}
As the geometric sum satisfies \(
\sum_{k=0}^{t-1} \tilde{\rho}^{k} \leq \frac{1}{1 - \tilde{\rho}}
\), then we get 
\begin{align*}
R_T^d \leq \frac{L_g \lambda_1 \lambda_2}{2} \left[ \frac{\tilde{\rho}}{1 - \tilde{\rho}} \| V_0 \| + \frac{1}{1 - \tilde{\rho}} \sum_{t=0}^{T-1} \|  b_{1,t} \| + \| b_2\|  \frac{T}{1 - \tilde{\rho}}  \right],
\end{align*}
which further simplifies to
\begin{align*}
R_T^d \leq \mathcal{O} \left( \| V_0 \| + \sum_{t=0}^{T-1} \|  b_{1,t}  \| + T \| b_2 \| \right).
\end{align*}
This completes the proof with $b_2 = \beta^2\sigma^2b_2'$.
\end{proof}

\begin{remark}
Existing studies have shown that, in general settings, the dynamic regret bound cannot achieve sublinear convergence in time \(T\)~\cite[]{li2022survey, eshraghi2022improving,Shahrampour2018,Notarnicola2023TAC,Li2021TAC,Dall2020,Mokhtari2016CDC}, which may explicitly depend on $P_T=\sum_{t=1}^{T-1}p_t$, the path
length related to the changes in the sequence of
minimizers. Moreover, some works depend on strong assumptions about objective functions. 
For example, ~\cite{eshraghi2022improving} establishes a bound of the form
\(
\mathcal{O}(1 + P_T),
\)
under the assumptions of strongly convex loss functions and bounded gradients. \cite{Shahrampour2018} gives a dynamic regret bound by $\mathcal{O}(\sqrt{(1 + C_T)T})$ with $C_T=\sum_{t=1}^{T}\|x_{t+1}^*-Ax_t^*\|$, requiring that the local time-varying functions have uniformly bounded gradients and the graph is undirected and connected.

In contrast, Theorem~\ref{th3} derives an upper bound on dynamic regret without the bounded gradient assumption under a stochastic setting and general time-varying digraphs. Due to the temporal variability of the gradients, the resulting bound incorporates additional error terms. Specifically, Theorem~\ref{th3} shows that the dynamic regret \(R^d_T\) consists of three components: a term dependent on initial conditions, a noise variance term induced by stochastic gradients, and an error that captures the time-varying nature of the problem, namely \(p_t\) and \(q_t\). In particular, the parameter $\beta$ can be properly tuned to reduce variance introduced by stochastic gradients. Moreover, if the temporal variations of both the optimal solution and the objective function's gradient decay sublinearly, and both the step size and the mixing parameter decrease over time, then the resulting dynamic regret can achieve sublinear convergence.
\end{remark}


Specifically, for the static distributed optimization with time-invariant functions ($f_t= f$), we can obtain a gradient-tracking based algorithm with variance reduction, as shown in the following corollary. 
\begin{corollary}\label{cor1}
For the static case with $f_t=f, t \geq 0$, when Assumptions~\ref{ass:strongtu}, \ref{ass:Lg}, \ref{ass:strong-connectivity}, \ref{ass:at}, \ref{ass:bt} hold and $\alpha$ satisfies \eqref{cond-alpha} with $m_0 = (1-\beta)^2$, it satisfies  
\begin{align*}
&  \limsup_{t \to \infty} V_t  \leq (\mathbb{I} -M(\alpha))^{-1}b, \\
&   \limsup_{t \to \infty} \mathbb{E}\left[\|\mathbf{x}_t - \hat{\mathbf{x}}_t\|_{\phi_t}^2\right] \leq   [(\mathbb{I} - M(\alpha))^{-1}b]_1,\\
&   \limsup_{t \to \infty} \mathbb{E}\left[\|\hat{x}_t - x^*\|^2\right] \leq  [(\mathbb{I} - M(\alpha))^{-1}b]_3,
\end{align*}
with a linear decay rate of $ \rho^t(M(\alpha))$, where $[u]_i$ denotes the $i$th entry of $u$ and $b= [0,\frac{2 n\tau^2 \psi}{1 - \tau}\beta^2  \sigma^2,0,2n \beta^2 \sigma^2]^{\top}$.
\end{corollary}
\begin{remark}
Corollary \ref{cor1} extends \cite{Nguyen2023} by incorporating the hybrid variance-reduction mechanism \eqref{zchange1}. 
As seen from the definition of $b$, the resulting error bounds in Corollary \ref{cor1} can be made arbitrarily small by reducing the parameter $\beta$, which highlights the effectiveness of the variance-reduction strategy. 
Furthermore, in contrast to the CTA-based gradient tracking framework employed in \cite{Nguyen2023} for static distributed optimization, our algorithm adopts an ATC framework adapted for online distributed optimization settings, which has been shown superior to CTA framework \cite[]{cattivelli2009diffusion,li2024npga}, particularly in terms of stability and convergence under dynamic conditions.
\end{remark}

\section{Numerical Examples}\label{Numberical}

In this section, we evaluate the effectiveness of the proposed TV-HSGT algorithm on two multi-agent distributed learning problems. The first problem is a distributed logistic regression task based on structured data, using the A9A dataset. The second problem is a distributed logistic regression task based on image data, using the MNIST dataset.
We compare the performance of the TV-HSGT algorithm with three baseline methods: DSGD \cite[]{lian2017can}, DSGT \cite[]{pu2021distributed}, and DSGT-HB \cite[]{Gao2023}. All methods adopt a unified strategy for constructing the communication weight matrices. Specifically, in each iteration of TV-HSGT, agents communicate over a time-varying strongly connected directed graph. This graph is constructed by randomly sampling edges from a predefined base directed graph while ensuring strong connectivity is maintained at each round. The communication mechanism follows the AB framework, employing a pair of row-stochastic and column-stochastic matrices for updating the decision and gradient tracking variables, respectively. The weights are uniformly distributed over each node’s in-neighbors or out-neighbors, making the implementation suitable for local computation. In contrast, the baseline methods DSGD, DSGT, and DSGT-HB operate over a fixed complete graph and assign uniform weights across all neighbors, forming symmetric doubly stochastic matrices.

\subsection{Distributed Logistic Regression on Structured Data}\label{subsec5.1}

This subsection evaluates the performance of the proposed TV-HSGT algorithm on a classification task using the structured A9A dataset with a logistic regression model. The loss function \cite[]{10806815} is defined as:
\vspace{-0.8em}  
\begin{small}
\begin{align*}
f(\theta, \xi^i) = \frac{1}{M^i} \sum_{s=1}^{M^i} \left( (1 - b^i_{s})(a^i_{s})^\top \theta - \log\left(s((a^i_{s})^\top \theta) \right) \right) + \frac{r^i}{2} \|\theta\|^2,
\end{align*}
\end{small}
where \(M^i\) is the number of samples for agent \(i\), \(r^i\) is a regularization coefficient, and \(s(a)\) denotes the sigmoid function.
We conduct two groups of experiments: (1) algorithm comparison and (2) parameter sensitivity analysis.

\begin{figure}[t]
\centering
\includegraphics[width=6.615cm]{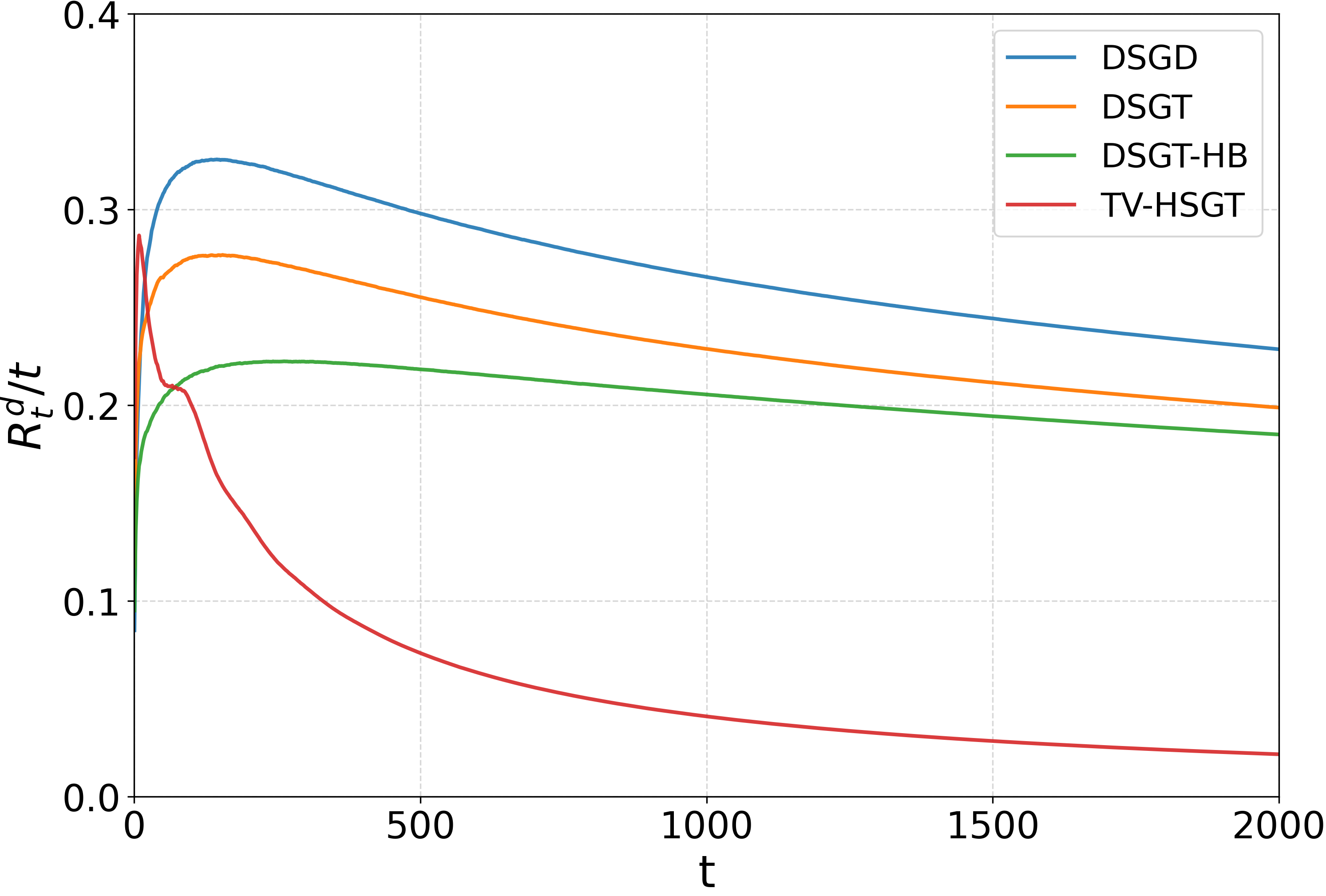}
\caption{Time-averaged regret on the A9A dataset for different algorithms}
\label{fig:3_a9a_compare_dy}
\end{figure}

\begin{figure}[t]
\centering
\includegraphics[width=6.615cm]{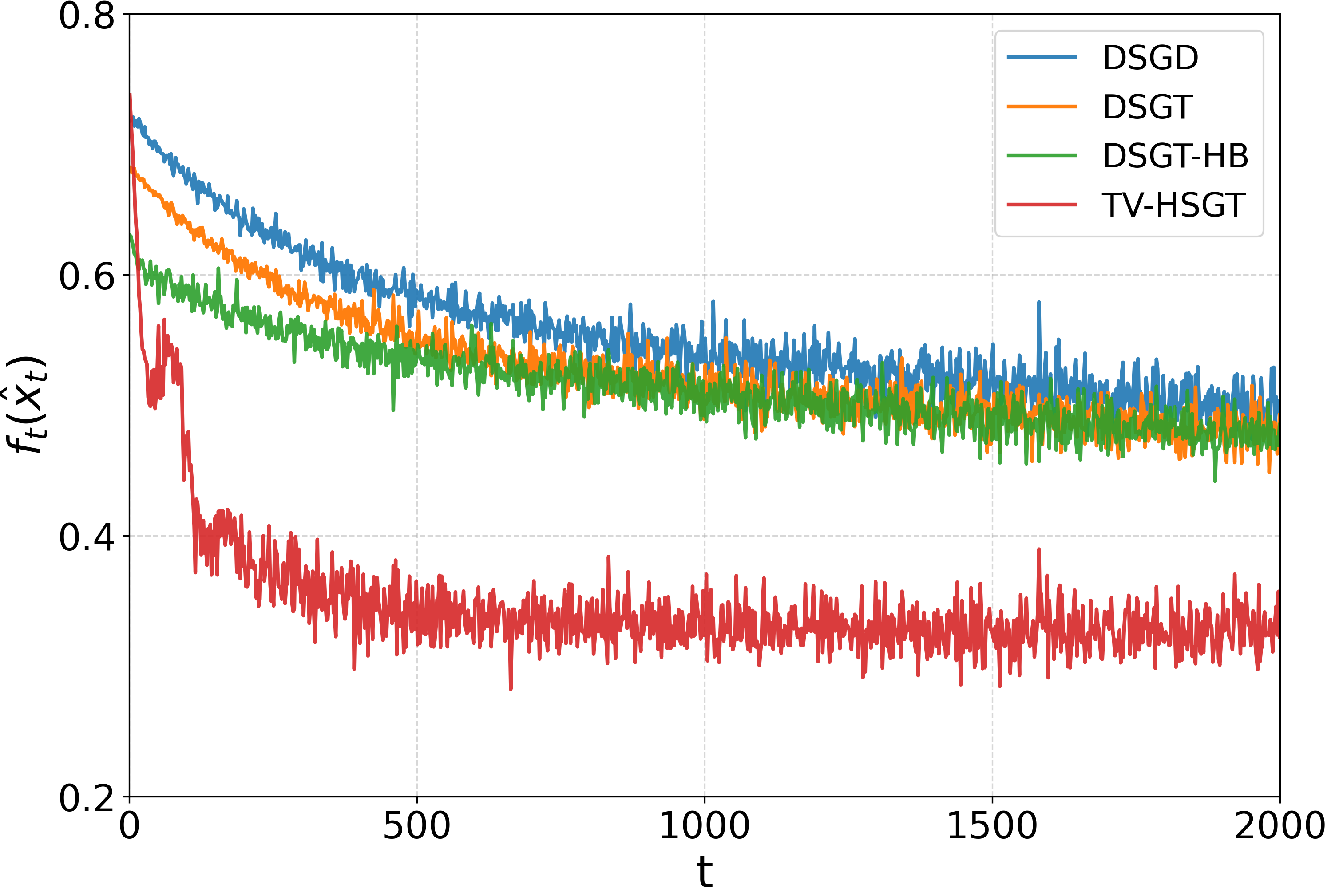}
\caption{Loss on the A9A dataset for different algorithms}
\label{fig:3_a9a_compare_loss}
\end{figure}
\begin{figure}[t]
\centering
\includegraphics[width=6.615cm]{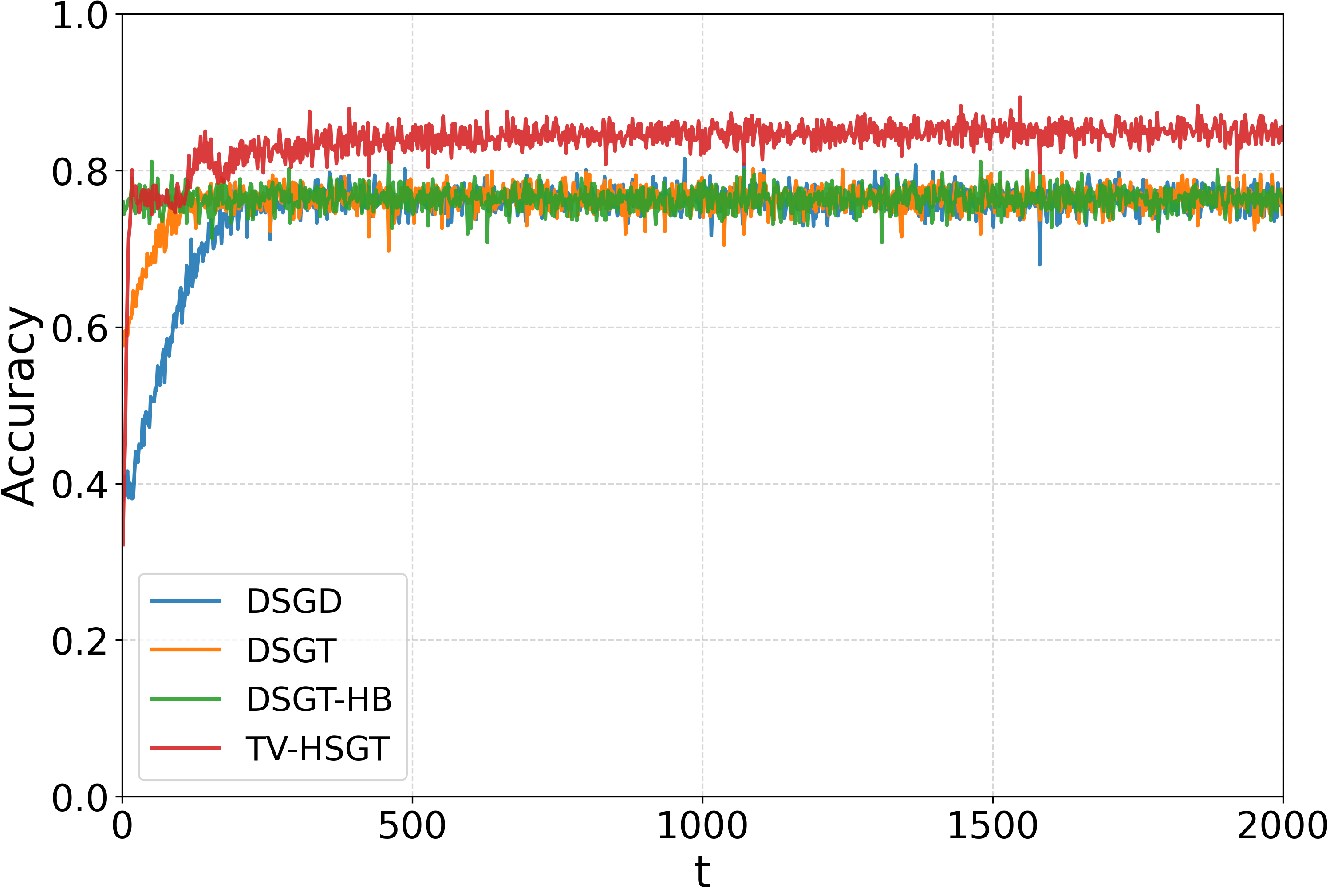}
\caption{Accuracy on the A9A dataset for different algorithms}
\label{fig:3_a9a_compare_acc}
\end{figure}

We compare TV-HSGT with the online versions of DSGD, DSGT, and DSGT-HB. Following the setup in \cite{10806815}, 10 agents independently receive mini-batches of 100 randomly drawn samples from the pre-shuffled A9A dataset at each round, simulating a dynamic online learning environment.
All methods use a fixed step size of 0.001. TV-HSGT adopts a mixing parameter \(\beta = 0.01\); DSGT-HB uses a momentum coefficient of 0.9; and regularization is set as \(r^i = 10^{-5}\) for all agents.
Figs.~\ref{fig:3_a9a_compare_dy}–\ref{fig:3_a9a_compare_acc} show that TV-HSGT consistently outperforms all baselines in terms of regret, loss, and accuracy. The hybrid variance reduction design effectively mitigates gradient noise and accelerates convergence, in line with the theoretical results in Theorem~\ref{th3}.

To examine the impact of the mixing parameter \(\beta\), we test values in \{0.01, 0.1, 0.2, 0.3, 0.4, 0.5\}. Figs.~\ref{fig:3_a9a_beta_beta}–\ref{fig:3_a9a_beta_acc} show that smaller \(\beta\) values lead to better performance, confirming the theoretical insights in Theorem~\ref{th3}. A larger \(\beta\) increases gradient noise, degrading performance.


\begin{figure}[t]
	\centering
	\includegraphics[width=6.615cm]{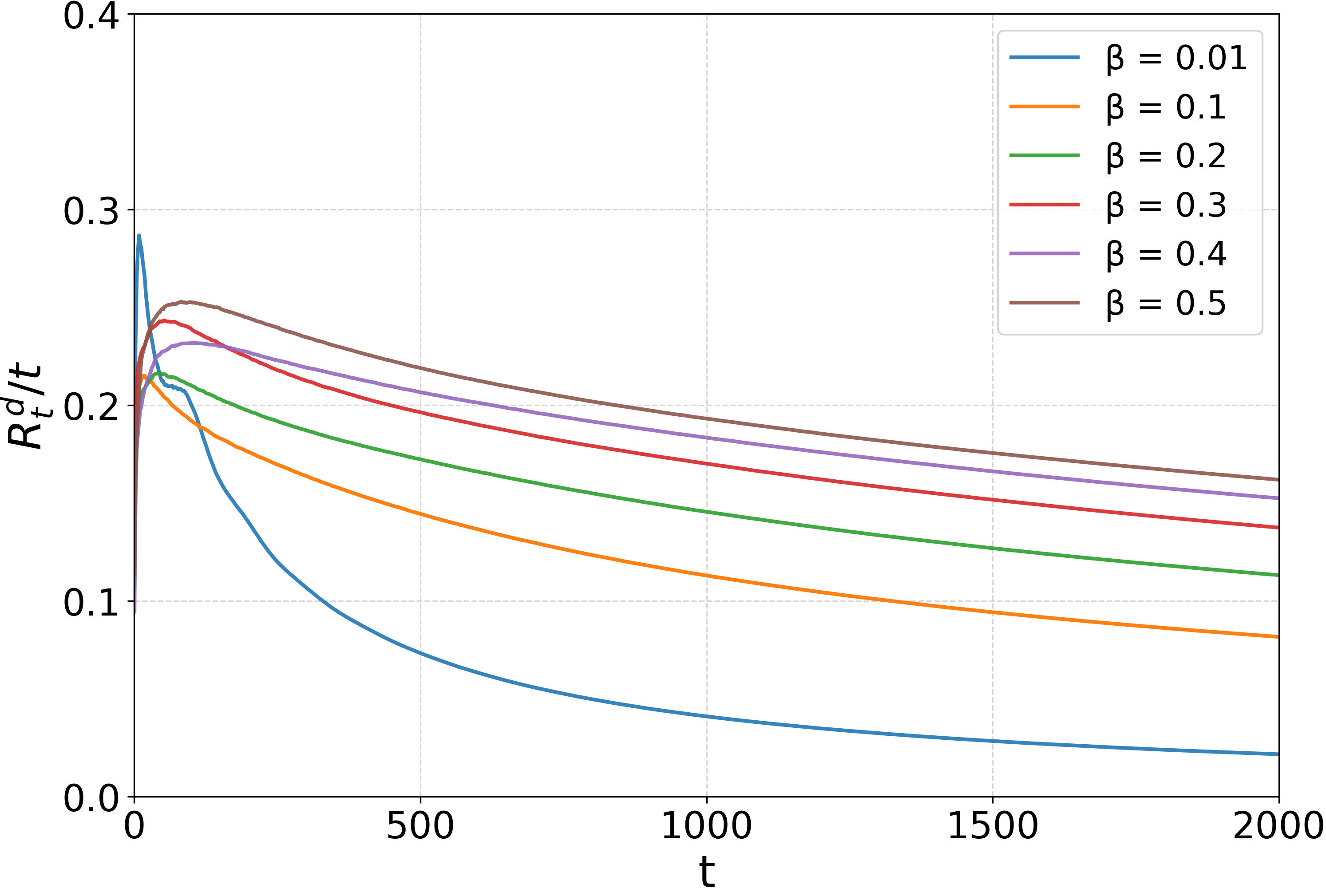}
	\caption{Time-averaged regret under different $\beta$ on the A9A dataset}
	\label{fig:3_a9a_beta_beta}
\end{figure}

\begin{figure}[t]
	\centering
	\includegraphics[width=6.615cm]{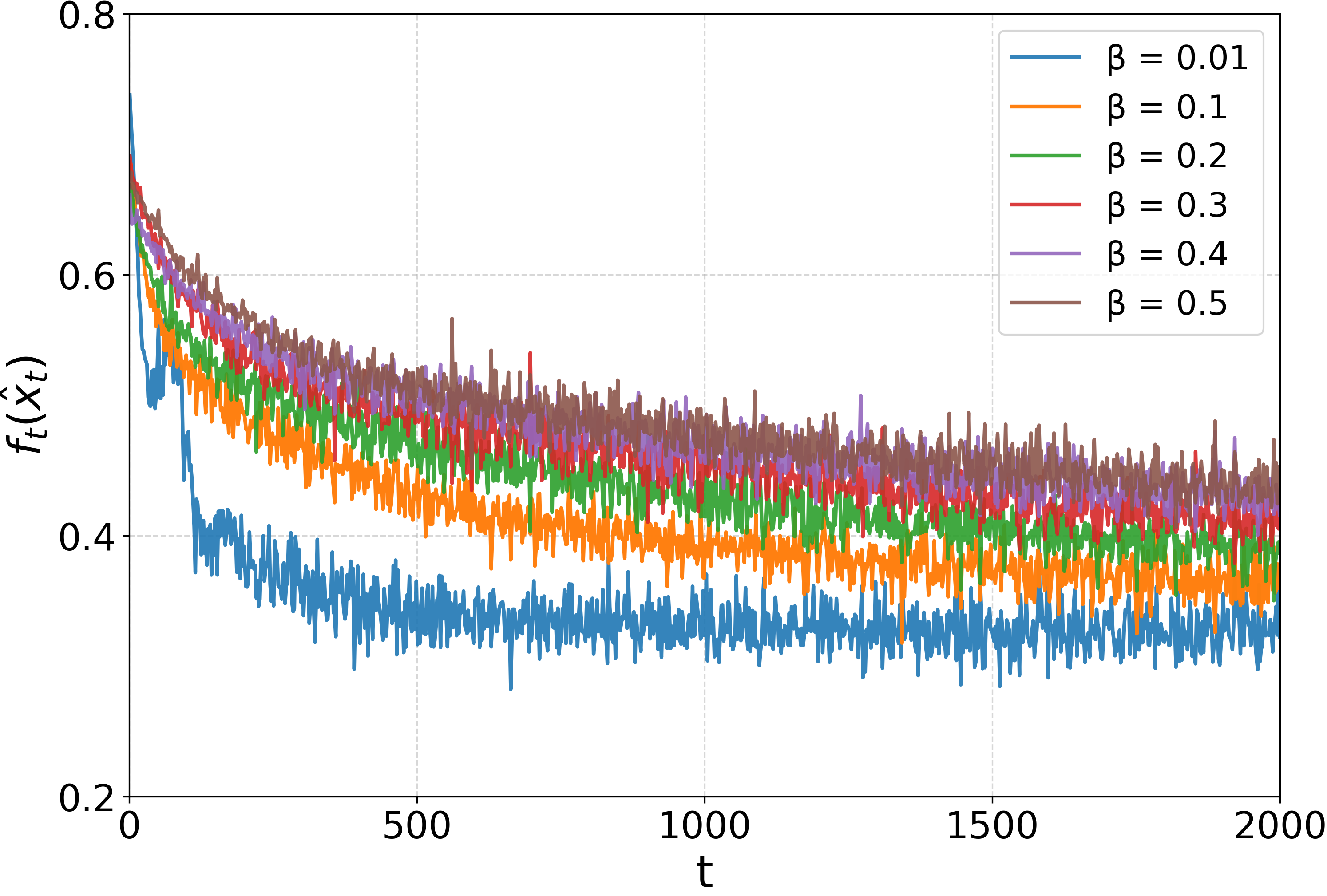}
	\caption{Loss under different $\beta$ values on the A9A dataset}
	\label{fig:3_a9a_beta_loss}
\end{figure}

\begin{figure}[t]
	\centering
	\includegraphics[width=6.615cm]{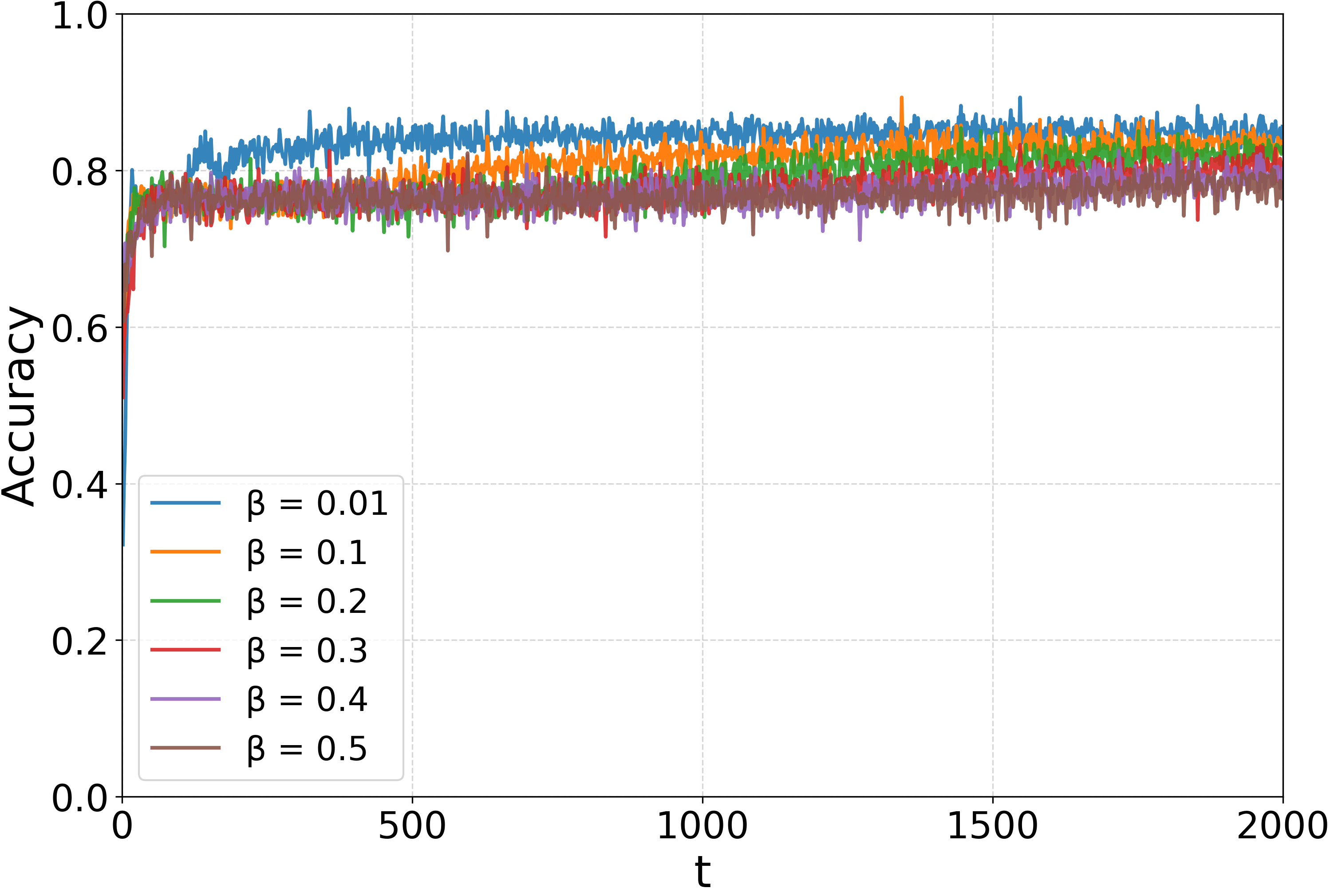}
	\caption{Accuracy under different $\beta$ values on the A9A dataset}
	\label{fig:3_a9a_beta_acc}
\end{figure}

\subsection{Distributed Logistic Regression on Image Data}
To further evaluate the effectiveness of TV-HSGT in visual settings, we conduct experiments on the MNIST dataset using a multi-class logistic regression model with $L_2$ regularization. The loss function is given by
\vspace{-0.8em}  
\begin{small}
\begin{align*}
f(\Theta, \xi^i) = \frac{1}{M^i} \sum_{s=1}^{M^i} \left( - \log \left( \frac{\exp\left(\theta_{b_s^i}^\top a_s^i\right)}{\sum_{k=0}^9 \exp\left(\theta_k^\top a_s^i\right)} \right) \right) + \frac{r^i}{2} \sum_{k=0}^{9} \|\theta_k\|^2,
\end{align*}
\end{small}
where \(\Theta = [\theta_0, \dots, \theta_9]\) is the parameter matrix, $a_s^i$ and $b_s^i$ represent the feature vector and label of sample $s$ at agent $i$, $M^i$ is the per-round batch size, and $r^i$ is the regularization coefficient.

All experimental settings match those of the structured-data experiments in Subsection \ref{subsec5.1}. 
Each agent processes 100 random images per round. Figs.~\ref{fig:3_mnist_compare_dy}–\ref{fig:3_mnist_compare_acc} show comparisons of time-averaged regret, loss, and accuracy across algorithms. The results demonstrate that TV-HSGT converges fastest, significantly reduces stochastic gradient noise, and achieves the highest final accuracy, outperforming DSGT-HB, DSGT, and DSGD—particularly in image classification applications.


\begin{figure}[t]
\centering
\includegraphics[width=6.615cm]{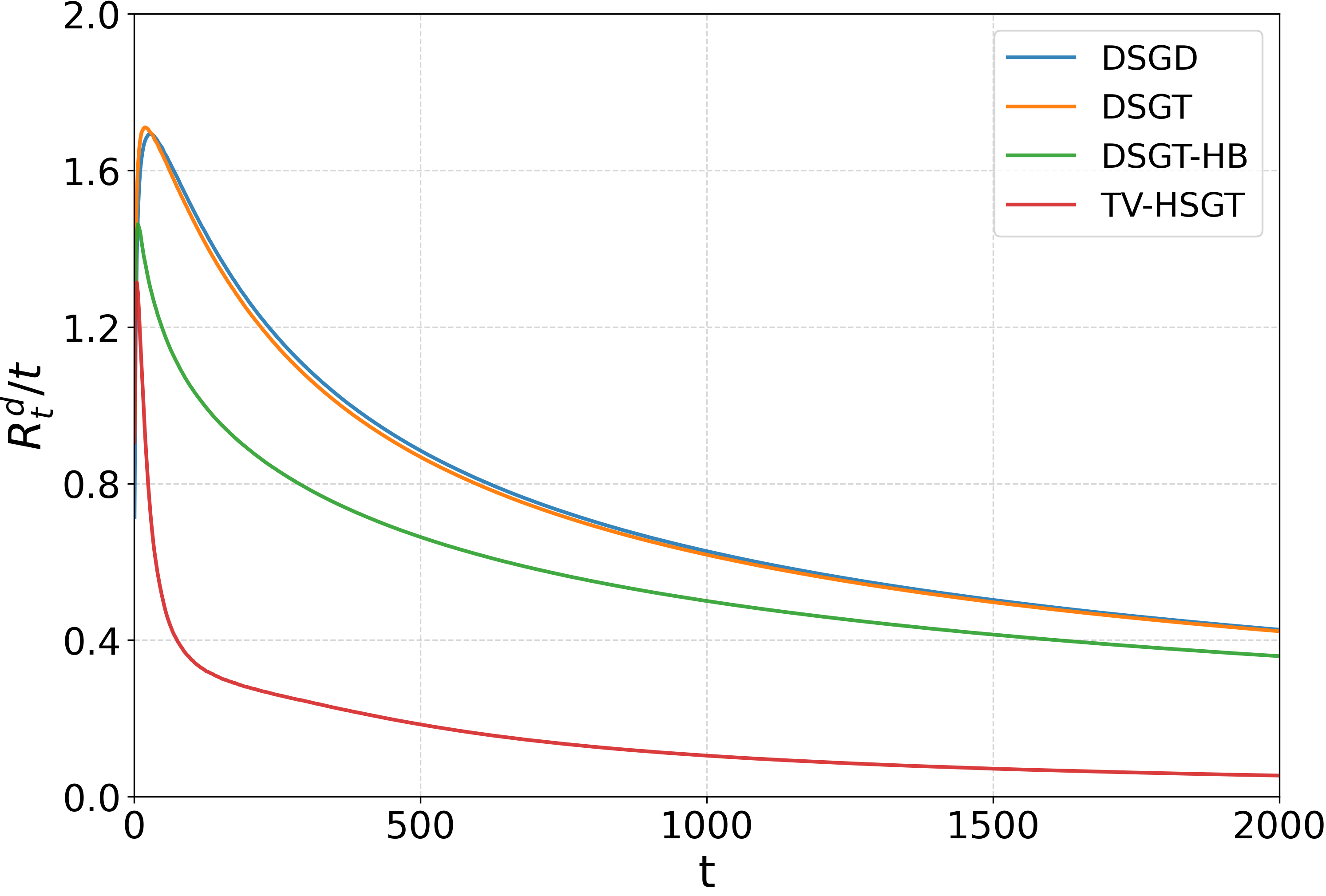}
\caption{Time-averaged regret on the MNIST dataset or different algorithms}
\label{fig:3_mnist_compare_dy}
\end{figure}
\begin{figure}[t]
\centering
\includegraphics[width=6.615cm]{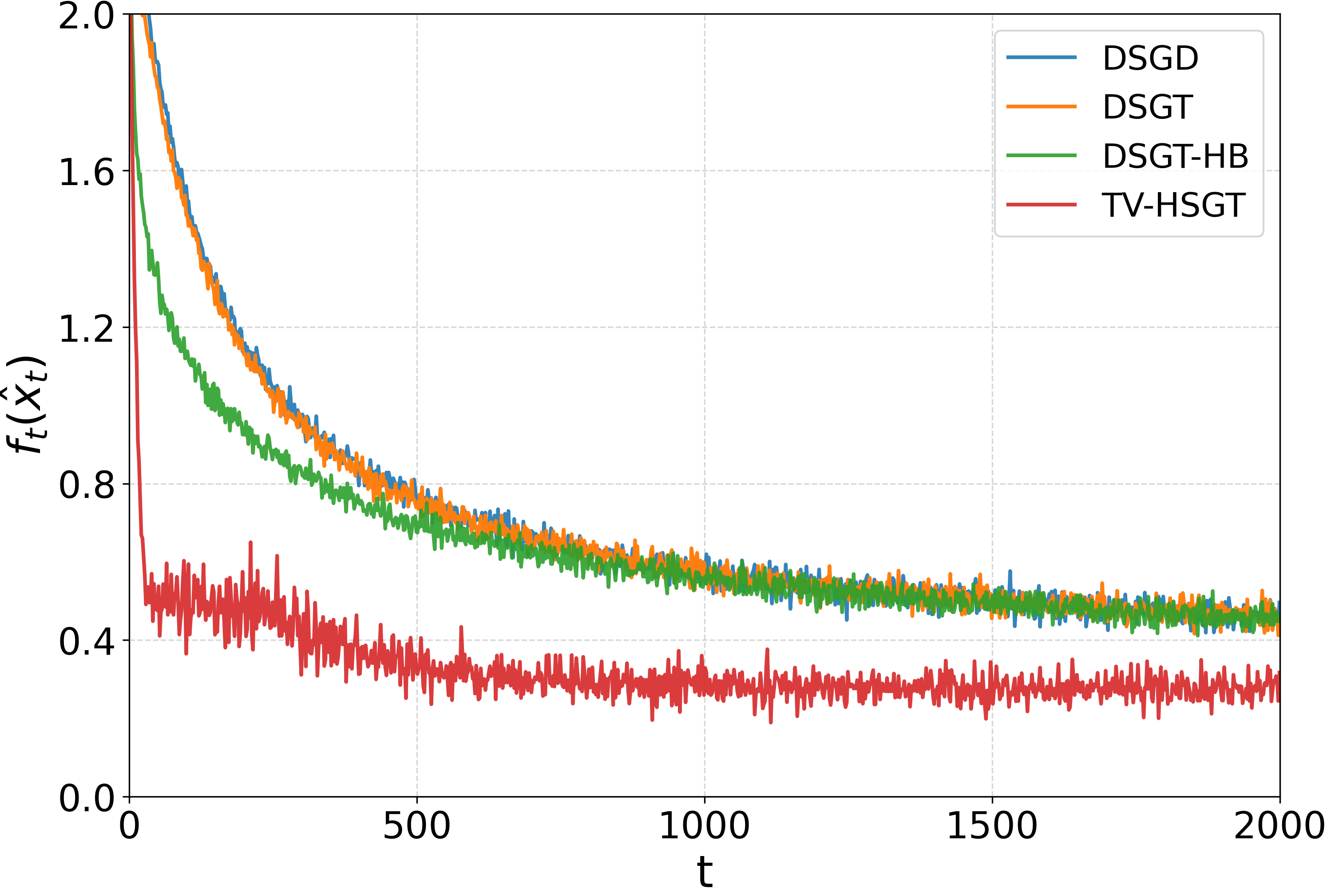}
\caption{Loss on the MNIST dataset for different algorithms}
\label{fig:3_mnist_compare_loss}
\end{figure}
\begin{figure}[t]
\centering
\includegraphics[width=6.615cm]{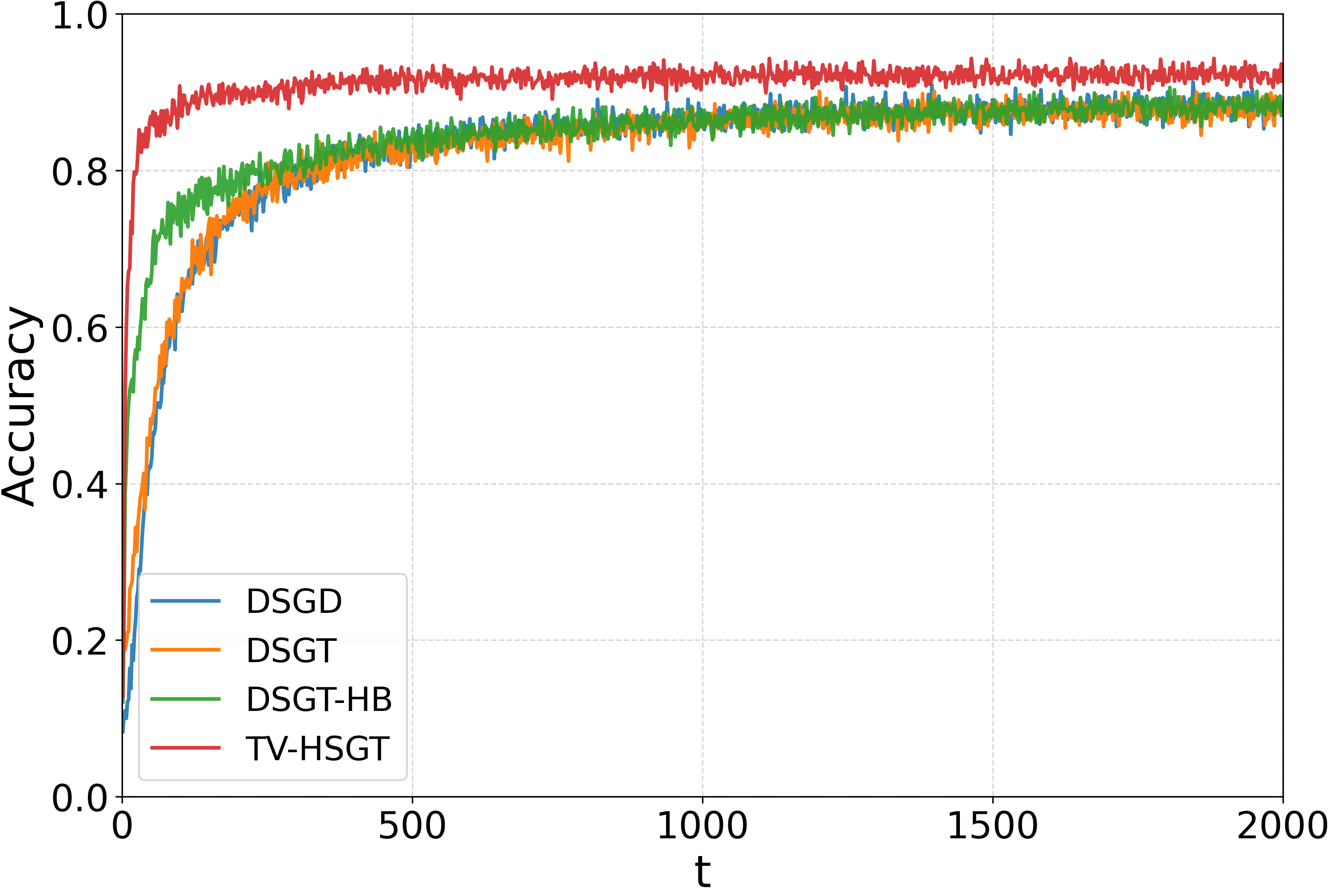}
\caption{Accuracy on the MNIST dataset for different algorithms}
\label{fig:3_mnist_compare_acc}
\end{figure}

We assess the effect of the mixing parameter $\beta \in \{0.01, 0.1, 0.2, 0.3, 0.4, 0.5\}$ on performance. Figs.~\ref{fig:3_mnist_beta_dy}–\ref{fig:3_mnist_beta_acc} illustrate that smaller $\beta$ values lead to better performance across regret, loss, and accuracy, consistent with our theoretical analysis in Theorem~\ref{th3}.

\section{Conclusion}\label{conclusion}
In this work, a novel decentralized online stochastic optimization algorithm named TV-HSGT has been proposed over time-varying directed networks with limited computation. By combining hybrid stochastic gradient estimation and gradient tracking strategies, an improved dynamic regret performance with variance reduction is achieved. An AB communication scheme is employed for a time-varying directed network to ensure consensus without eigenvector estimation. Theoretical analysis and experiments demonstrate the algorithm’s effectiveness in reducing variance and tracking the optimal solution. Future work will focus on improving the communication efficiency of TV-HSGT.

\begin{figure}[t]
\centering
\includegraphics[width=6.615cm]{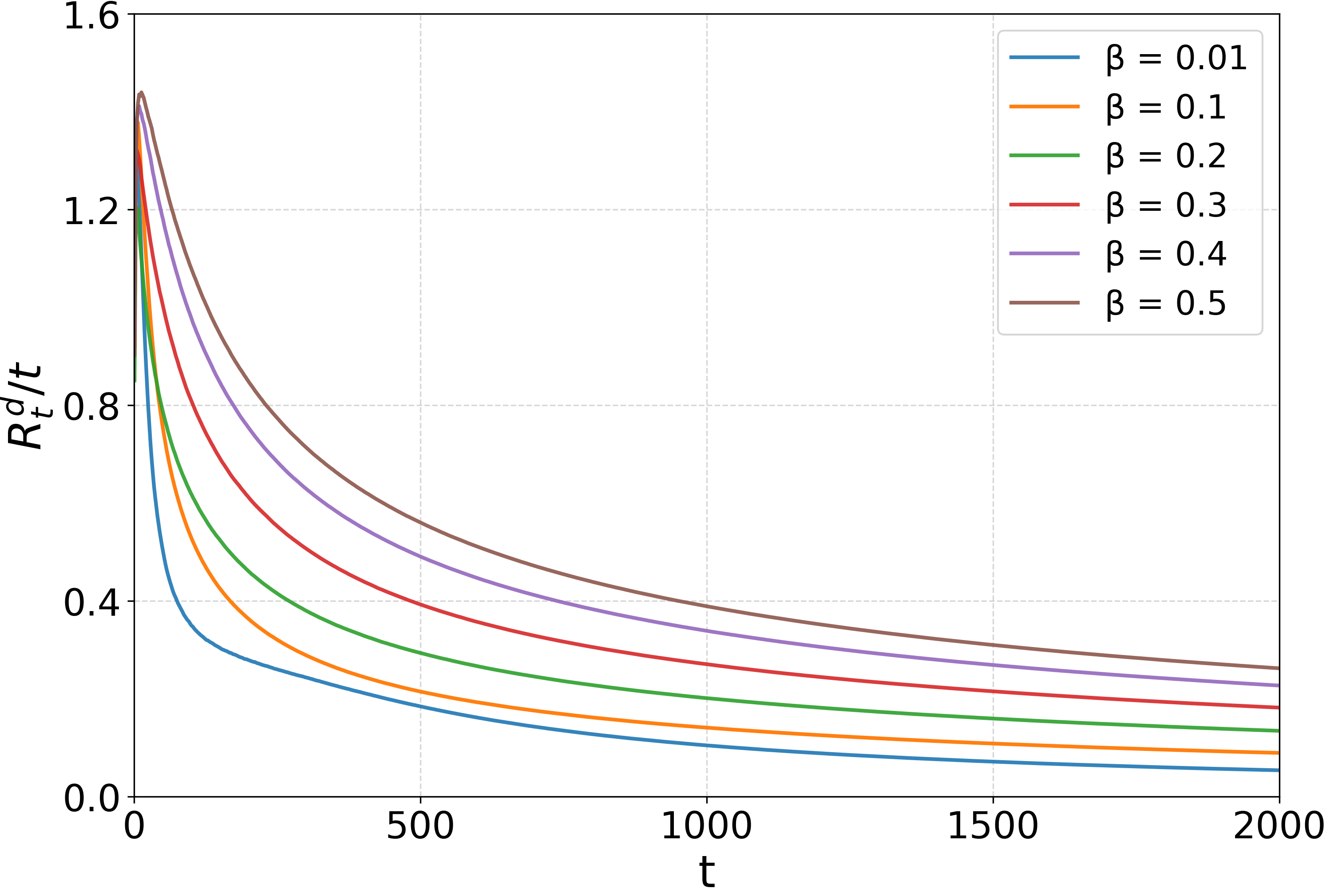}
\caption{Time-averaged regret under different $\beta$ values on the MNIST dataset}
\label{fig:3_mnist_beta_dy}
\end{figure}

\begin{figure}[t]
\centering
\includegraphics[width=6.615cm]{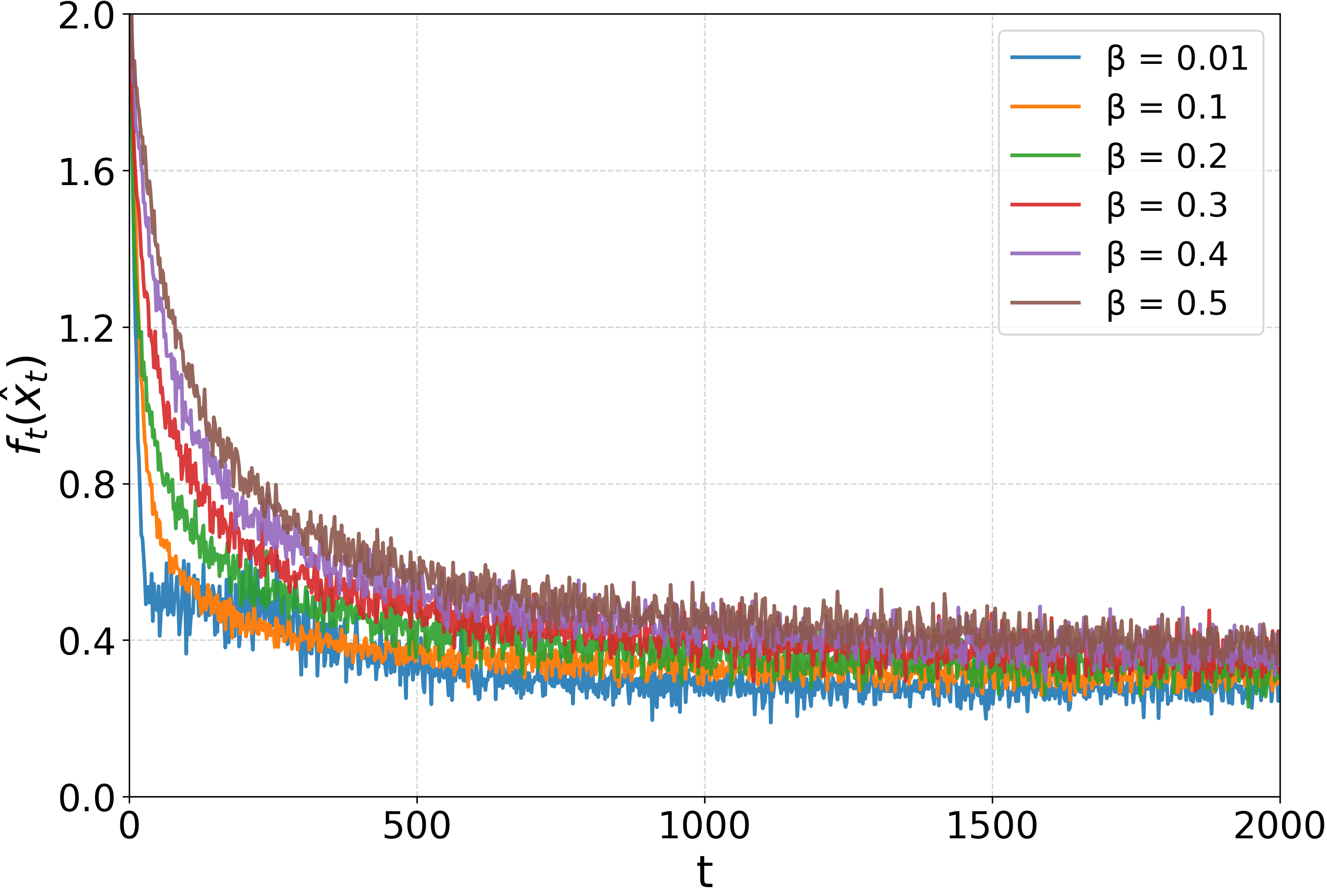}
\caption{Loss under different $\beta$ values on the MNIST dataset}
\label{fig:3_mnist_beta_loss}
\end{figure}

\begin{figure}[t]
\centering
\includegraphics[width=6.615cm]{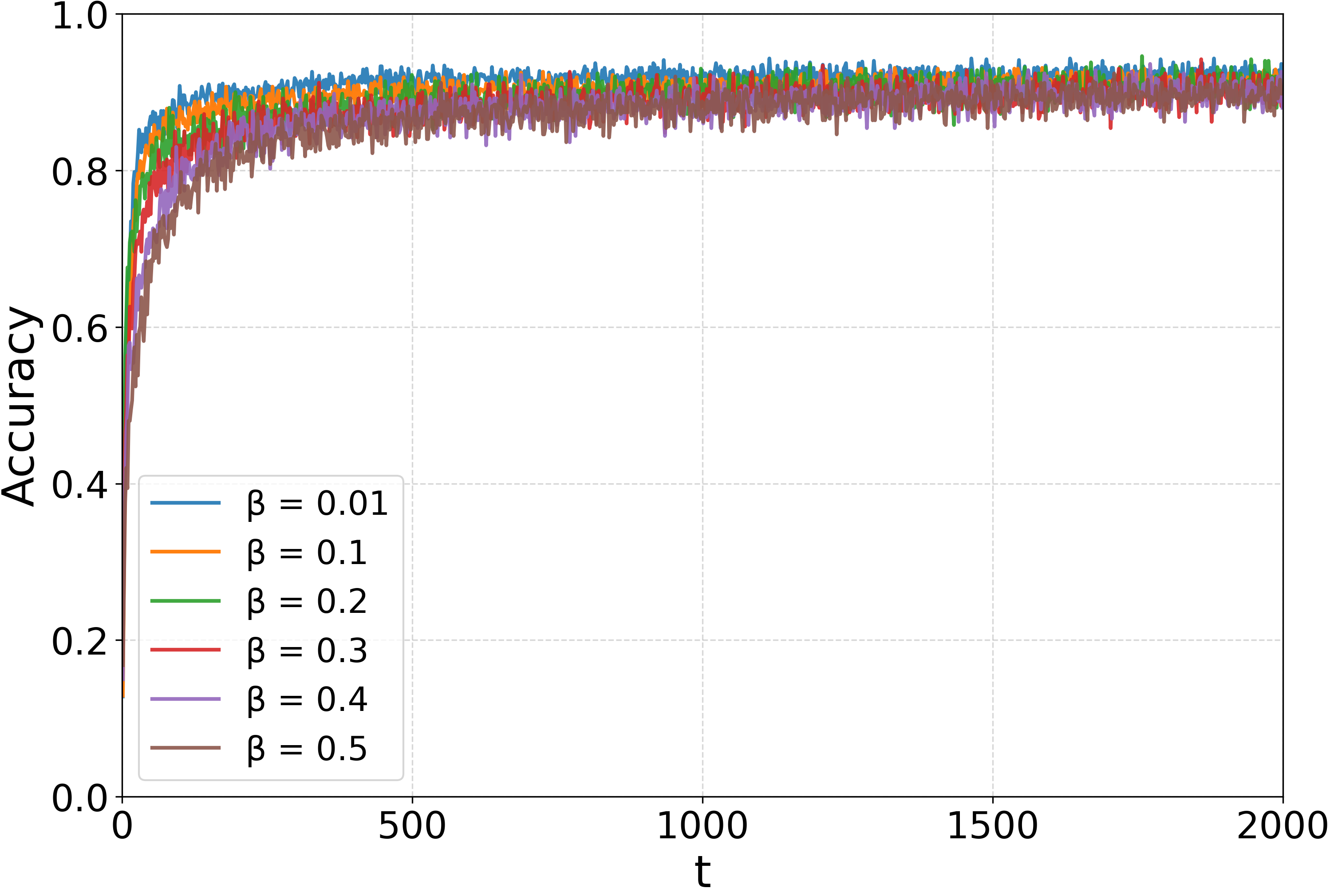}
\caption{Accuracy under different $\beta$ values on the MNIST dataset}
\label{fig:3_mnist_beta_acc}
\end{figure}

\section*{Appendix}
\appendix
\section{Proof of Lemma \ref{lemmayit}}

\begin{proof}
To bound \( \mathbb{E} \left[ \left \| \sum_{i=1}^n y_{i,t} \right\|^2 \right] \), we first apply the triangle inequality of norms to split \(\left\| \sum_{i=1}^n y_{i,t} \right\|\) as
\begin{align*} 
\left\| \sum_{i=1}^n y_{i,t} \right\| \leq \left\| \sum_{i=1}^n (y_{i,t} - \nabla f_{i,t}(x_{i,t})) \right\| + \left\| \sum_{i=1}^n \nabla f_{i,t}(x_{i,t}) \right\|.
\end{align*}
By the property of the global optimal solution \(x_t^*\), namely \(\sum_{i=1}^n \nabla f_{i,t}(x_t^*) = 0\), we obtain
\begin{align*} 
\left\| \sum_{i=1}^n \nabla f_{i,t}(x_{i,t}) \right\| & = \left\| \sum_{i=1}^n (\nabla f_{i,t}(x_{i,t}) - \nabla f_{i,t}(x_t^*)) \right\| \notag \\
& \leq \sum_{i=1}^n \left\| \nabla f_{i,t}(x_{i,t}) - \nabla f_{i,t}(x_t^*) \right\|.
\end{align*}
Since \(\nabla f_{i,t}\) is \(L_g\)-Lipschitz continuous, one has
\begin{align*} 
\left\| \nabla f_{i,t}(x_{i,t}) - \nabla f_{i,t}(x_t^*) \right\| \leq L_g \| x_{i,t} - x_t^* \|,
\end{align*}
which leads to
\begin{align}
\label{lemmaysima1}
\left\| \sum_{i=1}^n \nabla f_{i,t}(x_{i,t}) \right\| \leq L_g \sum_{i=1}^n \| x_{i,t} - x_t^* \| \leq L_g \sqrt{n} \varphi_t \| \mathbf{{x}}_t - \mathbf{x}_t^* \|_{\phi_t}.
\end{align}
By applying Lemma~\ref{lem:weighted-variance} with \(\gamma_i = \left[ \phi_{t} \right]_i\), \(u_i = x_{i,t}\), and \(\nu = x_t^*\), it can be derived that 
\begin{align*}
\|\mathbf{x}_t - \mathbf{x}_t^*\|_{\phi_t}^2 = \| \hat{{x}}_t - {x}_t^* \|^2 + \| \mathbf{x}_t - \hat{\mathbf{x}}_t \|_{\phi_t}^2.
\end{align*}
Noting that \(\sum_{i=1}^n y_{i,t} = \sum_{i=1}^n z_{i,t}\), we derive
\begin{align}
\label{lemmaysima2}
 \left\| \sum_{i=1}^n (y_{i,t} - \nabla f_{i,t}(x_{i,t})) \right\| &= \left\| \mathbf{1}_n^\top \otimes I_p (\mathbf{z}_t  -  \nabla F_t(\mathbf{x}_{t}))\right\| \notag \\ &\leq \sqrt{n} \left\|  (\mathbf{z}_t  -  \nabla F_t(\mathbf{x}_{t}))\right\|.
\end{align}
Combining \eqref{lemmaysima1} and \eqref{lemmaysima2}, it holds that
\begin{align*}
\left \| \sum_{i=1}^n y_{i,t} \right\|^2 & \leq 2n \left\| \mathbf{z}_t  -  \nabla F_t(\mathbf{x}_{t})\right\|^2 + 2 L_g^2 n \varphi_t^2 \| \mathbf{x}_t - \hat{\mathbf{x}}_t \|_{\phi_t}^2 \notag \\ & \quad+ 2 L_g^2 n \varphi_t^2\| \hat{x}_t - x_t^* \|^2.
\end{align*}
Taking the conditional expectation completes the proof.
\end{proof}

\section{Proof of Lemma \ref{le:ypai}}

\begin{proof}
Under the given assumptions, Lemma~\ref{lem:right-eigenvector} ensures that all components of the stochastic vector \( \pi_t \) are strictly positive. The scaling \( [\pi_t]_i^{-1} \) is therefore well-defined for all \( i \in \mathcal{V} \) and \( t \geq 0 \). By definition, we have
\begin{align*}
\| \mathbf{y}_t \|_{\pi_t^{-1}}^2 = \sum_{i=1}^{n} \frac{\| y_{i,t} \|^2}{[\pi_t]_i} = \sum_{i=1}^{n} [\pi_t]_i \left\| \frac{ y_{i,t} }{[\pi_t]_i}\right\|^2.
\end{align*}
Applying Lemma~\ref{lem:weighted-variance} with \( \gamma_i = [\pi_t]_i \), \( u_i = y_{i,t} / [\pi_t]_i \), and \( \nu = 0 \), it holds that
\begin{align*}
\sum_{i=1}^{n} [\pi_t]_i \left\| \frac{ y_{i,t} }{[\pi_t]_i}\right\|^2
=  S^2(\mathbf{y}_t, \pi_t) + \left\| \sum_{j=1}^{n}  y_{j,t} \right\|^2.
\end{align*}
Taking the conditional expectation on both sides and applying Lemma~\ref{lemmayit} completes the proof.
\end{proof}

\section{Proof of Lemma \ref{le:hattxing}}
\begin{proof}
According to the update rule in \eqref{xchange1}, it follows that \(\hat{x}_{t+1} = \hat{x}_t - \alpha \hat{y}_t\), so that \(\| \hat{x}_{t+1} - x_{t+1}^* \|^2 = \| \hat{x}_t - \alpha \hat{y}_t - x_{t+1}^* \|^2\). Introducing the auxiliary term \(\alpha n \phi_t^\top \pi_t \bar{y}_t\), where \(\bar{y}_t = \frac{1}{n} \sum_{j=1}^n y_{j,t}\), the error can be decomposed as
\begin{align}
& \hat{x}_t - \alpha \hat{y}_t - x_{t+1}^*  \notag \\ 
&= \underbrace{\hat{x}_t -  \alpha n \phi_{t}^\top \pi_t \nabla f_t(\hat{x}_t)-x_t^*}_{r_1} + \underbrace{\alpha \left( n \phi_{t}^{\top} \pi_t \bar{y}_t - \hat{y}_t \right)}_{r_4}  \notag \\
&\quad + \underbrace{\alpha n \phi_{t}^{\top} \pi_t h_t(\mathbf{x}_t)-\alpha n \phi_{t}^{\top} \pi_t \bar{y}_t}_{r_3} \notag 
+ \underbrace{x_t^* - x_{t+1}^*}_{r_5}  \notag \\
&\quad+ \underbrace{\alpha n \phi_{t}^\top \pi_t \nabla f_t(\hat{x}_t)-\alpha n \phi_{t}^{\top} \pi_t h_t(\mathbf{x}_t)}_{r_2} .
\end{align}
Applying Lemma~\ref{lemmazeta}, the following inequality holds
\begin{align}
\| \hat{x}_{t+1} - x_{t+1}^* \|^2 & \leq \zeta \| r_1  \|^2 + \frac{4\zeta}{\zeta - 1} \| r_2  \|^2 + \frac{4\zeta}{\zeta - 1} \| r_3  \|^2 \notag \\
&\quad+ \frac{4\zeta}{\zeta - 1} \| r_4  \|^2 + \frac{4\zeta}{\zeta - 1} \| r_5  \|^2.
\end{align}
Since \(f_t\) is \(\mu\)-strongly convex, Lemma~\ref{lemma:mu} implies that if the step size satisfies \(0 < \alpha < \frac{2}{n(\mu +L_g)\phi_{t}^\top \pi_t}\), then
\(
\| r_1 \|^2  \leq (1 - \mu \alpha n \phi_{t}^\top \pi_t)^2\| \hat{x}_t - x_t^* \|^2.
\)
By Lemma \ref{lem:average-gradient-diff}, we obtain
\(
\| r_2 \|^2  \leq  n \alpha^2 (\phi_{t}^\top \pi_t)^2 L_g^2 \varphi_t^2 \| \mathbf{x}_t - \hat{\mathbf{x}}_t \|_{\phi_t}^2.
\)
Since \(\bar{y}_t = \bar{z}_t\) and based on the definition of the gradient tracking error, it holds that
\[
\mathbb{E} \left[ \| r_3 \|^2  \right]  \leq \alpha^2 n(\phi_{t}^\top \pi_t)^2\mathbb{E} \left[ \| \mathbf{z}_{t} - \nabla F_{t}(\mathbf{x}_{t}) \|^2 \right].
\]
Applying Lemma~\ref{lem:weighted-variance} with \(u_i = [\pi_t]_i \left( \frac{y_{i,t}}{[\pi_t]_i} - \sum_{j=1}^{n} y_{j,t} \right)\), \(\gamma_i = [\phi_t]_i\), and \(\nu = 0\), we obtain
\[
\| r_4 \|^2  
\leq \alpha^2\sum_{i=1}^{n} \left[ \pi_t \right]_{i} \left\| \frac{y_{i,t}}{\left[ \pi_t \right]_i} - \sum_{j=1}^{n} y_{j,t} \right\|^2.
\]
Therefore, from the definition of \(S^2(\mathbf{y}_t, \pi_t)\) in \eqref{eq:grad-track-error}, we have
\[
\mathbb{E} \left[ \| r_4 \|^2  \right] \leq \alpha^2 \mathbb{E} \left[ S^2 (\mathbf{y}_t, \pi_t)  \right].
\]
Combining the results above, and under the condition that \(0 < \alpha < \frac{2}{n(\mu +L_g)\phi_{t}^\top \pi_t}\), we have
\begin{align}
& \mathbb{E} \left[ \| \hat{x}_{t+1} - x_{t+1}^* \|^2  \right] \notag \\
&\leq \zeta(1 - \mu \alpha n \phi_{t}^\top \pi_t)^2 \mathbb{E} \left[ \| \hat{x}_t - x_t^* \|^2  \right]  + \frac{4\zeta}{\zeta - 1} \alpha^2 \mathbb{E} \left[ S^2(\mathbf{y}_t, \pi_t) \right] \notag \\
& + \frac{4\zeta}{\zeta - 1} \| x_t^* - x_{t+1}^* \|^2 + \frac{4\zeta}{\zeta - 1} \alpha^2 n(\phi_{t}^\top \pi_t)^2 \mathbb{E} \left[\left \|  \mathbf{z}_t  -  \nabla F_t(\mathbf{x}_{t})\right\|^2\right] \notag \\
&+ \frac{4\zeta}{\zeta - 1} n \alpha^2 (\phi_{t}^\top \pi_t)^2 L_g^2 \varphi_t^2 \mathbb{E} \left[\| \mathbf{x}_t - \hat{\mathbf{x}}_t \|_{\phi_t}^2\right].
\end{align}
Finally, choosing \(\zeta = \frac{1}{1 - \mu \alpha n \phi_{t}^\top \pi_t}\) ensures convergence and completes the proof.
\end{proof}

\section{Proof of Lemma \ref{xtxhat}}
\begin{proof}
Since \(\hat{\mathbf{x}}_{t+1} = \hat{\mathbf{x}}_t - \alpha \hat{\mathbf{y}}_t\) and \(\mathbf{x}_{t+1} = A_t\mathbf{x}_t - \alpha A_t\mathbf{y}_t\), it follows that \(\mathbf{x}_{t+1} - \hat{\mathbf{x}}_{t+1} = (A_t\mathbf{x}_t - \hat{\mathbf{x}}_t) - \alpha(A_t\mathbf{y}_t - \hat{\mathbf{y}}_t)\). Taking the \({\phi_{t+1}}\)-norm on both sides and applying Lemma~\ref{lemmazeta}, we obtain
\vspace{-1em}  
\begin{small}
\begin{align*}
\|\mathbf{x}_{t+1} - \hat{\mathbf{x}}_{t+1}\|_{\phi_{t+1}}^2 = \zeta \|A_t\mathbf{x}_t - \hat{\mathbf{x}}_t\|_{\phi_{t+1}}^2 + \frac{\zeta \alpha^2}{\zeta - 1} \|A_t\mathbf{y}_t - \hat{\mathbf{y}}_t\|_{\phi_{t+1}}^2.
\end{align*}
\end{small}
Both terms \(\|A_t\mathbf{x}_t - \hat{\mathbf{x}}_t\|_{\phi_{t+1}}^2\) and \(\|A_t\mathbf{y}_t - \hat{\mathbf{y}}_t\|_{\phi_{t+1}}^2\) conform to the structure of Lemma~\ref{le:phi}, with \(A = A_t\) and \(x_i = x_{i,t}\) for all \(i \in \mathcal{V}\). In addition, Lemma~\ref{lem:left-eigenvector} implies that \(\phi_{t+1}^\top A_t = \phi_t^\top\). Letting \(\pi = \phi_{t+1}\), \(\phi = \phi_t\), and \(\hat{x}_\phi = x_t\), and substituting into Lemma~\ref{le:phi}, we obtain
\(
\|A_t\mathbf{x}_t - \hat{\mathbf{x}}_t\|_{\phi_{t+1}}^2 \leq c_t^2 \|\mathbf{x}_t - \hat{\mathbf{x}}_t\|_{\phi_t}^2.
\)
Using the upper bound of \(c_t\), this gives
\begin{equation}
\label{axxhat}
\|A_t\mathbf{x}_t - \hat{\mathbf{x}}_t\|_{\phi_{t+1}}^2 \leq c^2 \|\mathbf{x}_t - \hat{\mathbf{x}}_t\|_{\phi_t}^2.
\end{equation}
Similarly, it follows that
\begin{equation}
\label{yxyhat}
\|A_t\mathbf{y}_t - \hat{\mathbf{y}}_t\|_{\phi_{t+1}}^2 \leq c^2 \|\mathbf{y}_t - \hat{\mathbf{y}}_t\|_{\phi_t}^2.
\end{equation}
To bound \(\|\mathbf{y}_t - \hat{\mathbf{y}}_t\|_{\phi_t}^2\), we apply Lemma~\ref{lem:weighted-variance} with \(\gamma_i = [\phi_t]_i\), \(u_i = y_{i,t}\), and \(\nu = 0\). Then, we have
\begin{align}\label{y-diff-phi}
\|\mathbf{y}_t - \hat{\mathbf{y}}_t\|_{\phi_t}^2 
&= \sum_{i=1}^n [\phi_t]_i \|y_{i,t} - \sum_{j=1}^n [\phi_t]_j y_{j,t}\|^2 \notag \\
&\leq \sum_{i=1}^n [\phi_t]_i \|y_{i,t}\|^2 = \sum_{i=1}^n [\phi_t]_i [\pi_t]_i \frac{\|y_{i,t}\|^2}{[\pi_t]_i} \notag \\
&\leq \gamma_t^2 \| \mathbf{y}_t \|_{\pi_t^{-1}}^2.
\end{align}
where \(\gamma_t = \sqrt{\max_{i \in \mathcal{V}} \left( [\phi_t]_i [\pi_t]_i \right)}\), and \(\| \mathbf{y}_t \|_{\pi_t^{-1}}^2 = \sum_{i=1}^n \frac{\|y_{i,t}\|^2}{[\pi_t]_i}\). Therefore,
\[
\|\mathbf{x}_{t+1} - \hat{\mathbf{x}}_{t+1}\|_{\phi_{t+1}}^2 = \zeta c^2 \|\mathbf{x}_t - \hat{\mathbf{x}}_t\|_{\phi_t}^2 + \frac{\zeta \alpha^2 c^2 \gamma_t^2}{\zeta - 1} \| \mathbf{y}_t \|_{\pi_t^{-1}}^2.
\]
Letting \(\zeta = \frac{1 + c^2}{2 c^2}\), we obtain
\begin{small} 
\begin{align*}
\|\mathbf{x}_{t+1} - \hat{\mathbf{x}}_{t+1}\|_{\phi_{t+1}}^2 \leq \frac{1 + c^2}{2} \|\mathbf{x}_t - \hat{\mathbf{x}}_t\|_{\phi_t}^2 + \frac{1 + c^2}{1 - c^2} \alpha^2 c^2 \gamma_t^2 \| \mathbf{y}_t \|_{\pi_t^{-1}}^2.
\end{align*}
\end{small}
Taking the conditional expectation and applying Lemma~\ref{le:ypai} completes the proof.
\end{proof}

\section{Proof of Lemma \ref{xt1xt}}
\begin{proof}
By adding and subtracting \(\hat{\mathbf{x}}_t\), we obtain
\(
\|\mathbf{x}_{t+1} - \mathbf{x}_{t}\|
= \|\mathbf{x}_{t+1} - \hat{\mathbf{x}}_{t} + \hat{\mathbf{x}}_{t} - \mathbf{x}_{t}\|
\le \|A_{t}\mathbf{x}_{t} - \hat{\mathbf{x}}_{t}\|
+ \|\mathbf{x}_{t} - \hat{\mathbf{x}}_{t}\|
+ \alpha \| A_{t}\mathbf{y}_{t}\|,
\)
where the inequality follows from the update rule of \(x\) in Equation~\eqref{xchange1} and the triangle inequality. Expanding the norms and applying Lemma \ref{le:phi} yield
\begin{align*}
& \|\mathbf{x}_{t+1} - \mathbf{x}_{t}\|\notag \\
& \leq \varphi_{t+1} \|A_{t}\mathbf{x}_{t} - \hat{\mathbf{x}}_{t}\|_{\phi_{t+1}}
+ \varphi_{t} \|\mathbf{x}_{t} - \hat{\mathbf{x}}_{t}\|_{\phi_{t}}
+ \alpha \|A_{t}\mathbf{y}_{t}\| \notag \\
& \leq (c\varphi_{t+1}+\varphi_{t}) \|\mathbf{x}_{t} -\hat{\mathbf{x}}_{t}\|_{\phi_{t}} + \alpha \|A_{t}\mathbf{y}_{t}\|.
\end{align*}
Using inequality~\eqref{yxyhat}, \eqref{y-diff-phi} and the definition \(\gamma_t= \sqrt{ \max_{i} [\phi_t]_i [\pi_t]_i }\), we obtain
\begin{align*}
\|A_{t}\mathbf{y}_{t}\|  & \leq \|A_{t}\mathbf{y}_{t}-\hat{\mathbf{y}}_{t}\| +\|\hat{\mathbf{y}}_{t}\| \notag \\
& \leq \varphi_{t+1}\|A_{t}\mathbf{y}_{t}-\hat{\mathbf{y}}_{t}\|_{\phi_{t+1}}+\|\hat{\mathbf{y}}_{t}\| \notag \\
& \leq c \varphi_{t+1}\|\mathbf{y}_t-\hat{\mathbf{y}}_t\|_{\phi_{t}} + \gamma_t \| \mathbf{y}_t \|_{\pi_t^{-1}} \notag \\
&\leq c \gamma_t \varphi_{t+1} \| \mathbf{y}_t \|_{\pi_t^{-1}} + \gamma_t \| \mathbf{y}_t \|_{\pi_t^{-1}}.
\end{align*}
By employing the norm inequality \(\|A_t\mathbf{x}_t - \hat{\mathbf{x}}_t\|_{\phi_{t+1}} \leq c \|\mathbf{x}_t - \hat{\mathbf{x}}_t\|_{\phi_t}\) as given in Equation~(\ref{axxhat}) and invoking Lemma~\ref{lemmazeta}, we derive
\begin{align*}
\|\mathbf{x}_{t+1} - \mathbf{x}_{t}\|^2
& \leq 2(c\varphi_{t+1}+\varphi_{t})^2 \|\mathbf{x}_{t} -\hat{\mathbf{x}}_{t}\|_{\phi_{t}}^2
\\ & \quad + 2 \alpha^2 \gamma_t^2 (c \varphi_{t+1} + 1)^2 \| \mathbf{y}_t \|_{\pi_t^{-1}}^2.
\end{align*}
Taking expectation on both sides and applying the bound from Lemma~\ref{le:ypai} yields the desired result.
\end{proof}

\section{Proof of Lemma \ref{zt1zt}}
\begin{proof}
Based on the update rule of the hybrid stochastic gradient estimator given in Equation~\eqref{zchange1}, the update difference between \( z_{i,t+1} \) and \( z_{i,t} \) can be expressed as
\begin{align*}
z_{i,t+1} - z_{i,t} &= \nabla \hat{f}_{i,t+1}(x_{i,t+1}, \xi_{i,t+1}) - \nabla \hat{f}_{i,t+1}(x_{i,t}, \xi_{i,t+1})  \notag \\
& \quad + \beta (\nabla \hat{f}_{i,t+1}(x_{i,t}, \xi_{i,t+1}) - \nabla f_{i,t}(x_{i,t})) \notag \\
& \quad - \beta (z_{i,t} - \nabla f_{i,t}(x_{i,t})).
\end{align*}

Applying the norm inequality and Lemma~\ref{lemmazeta}, we decompose \(\| z_{i,t+1} - z_{i,t} \|^2\) into three terms
\vspace{-0.8em}  
\begin{small} 
\begin{align*}
\| z_{i,t+1} - z_{i,t} \|^2 & \leq 3 \| \nabla \hat{f}_{i,t+1}(x_{i,t+1}, \xi_{i,t+1}) - \nabla \hat{f}_{i,t+1}(x_{i,t}, \xi_{i,t+1}) \|^2 \notag\\
& \quad + 3 \beta^2 \| \nabla \hat{f}_{i,t+1}(x_{i,t}, \xi_{i,t+1}) - \nabla f_{i,t}(x_{i,t}) \|^2 \notag\\
& \quad + 3 \beta^2 \| z_{i,t} - \nabla f_{i,t}(x_{i,t}) \|^2.
\end{align*}
\end{small}

From Assumption~\ref{ass:Lg}, the stochastic gradient \(\nabla \hat{f}_{i,t+1}(\cdot, \xi_{i,t+1})\) is \(L_g\)-Lipschitz continuous, and hence
\(
\mathbb{E} \left[ \| \nabla \hat{f}_{i,t+1}(x_{i,t+1}, \xi_{i,t+1}) - \nabla \hat{f}_{i,t+1}(x_{i,t}, \xi_{i,t+1}) \|^2 \right] \leq L_g^2 \mathbb{E} \left[ \| x_{i,t+1} - x_{i,t} \|^2 \right].
\)

Furthermore, decomposing the variance of stochastic gradients and temporal variation yields
\begin{align*}
& \mathbb{E} \left[ \| \nabla \hat{f}_{i,t+1}(x_{i,t}, \xi_{i,t+1}) - \nabla f_{i,t}(x_{i,t}) \|^2 \right] \notag\\
&\leq 2 \mathbb{E} \left[ \| \nabla \hat{f}_{i,t+1}(x_{i,t}, \xi_{i,t+1}) - \nabla f_{i,t+1}(x_{i,t}) \|^2 \right] \notag \\
&\quad + 2 \mathbb{E} \left[ \| \nabla f_{i,t+1}(x_{i,t}) - \nabla f_{i,t}(x_{i,t}) \|^2 \right] \notag \\
&\leq 2 \sigma^2 + 2 q_t^2,
\end{align*}
where \(\sigma^2\) denotes the variance from the stochastic gradients due to Assumption \ref{ass:sigma}, and \(q_t\) is defined in \eqref{qt}.

Combining the bounds above, we obtain
\begin{align*}
\label{usezt1zt}
\mathbb{E} \left[ \| \mathbf{z}_{t+1} - \mathbf{z}_{t} \|^2  \right] &\leq 3 L_g^2 \mathbb{E} \left[ \| \mathbf{x}_{t+1} - \mathbf{x}_{t} \|^2 \right] + 6 \beta^2 n q_t^2 \notag\\
& \quad + 3 \beta^2 \mathbb{E} \left[ \| \mathbf{z}_{t} - \nabla F_{t}(\mathbf{x}_{t}) \|^2 \right] + 6 \beta^2 n \sigma^2.
\end{align*}

Substituting the bound from Lemma~\ref{xt1xt} into the expression completes the proof.
\end{proof}

\section{Proof of Lemma \ref{le:yt1pait1}}
\begin{proof}
Since \(B_t\) is a column-stochastic matrix, the update rule of the gradient tracking variable can be written compactly as 
\[\mathbf{y}_{t+1} = B_t\mathbf{y}_{t}+B_t\mathbf{z}_{t+1}-B_t\mathbf{z}_{t}. \] By multiplying both sides with \(\text{diag}^{-1}(\pi_{t+1})\) and subtracting the state \(\mathbf{s}_{t+1} = \mathbf{1}_n \mathbf{1}_n^\top \mathbf{y}_{t+1}= \mathbf{s}_t +\mathbf{1}_n \mathbf{1}_n^\top (\mathbf{z}_{t+1} - \mathbf{z}_t)\), we obtain
\begin{align*}
& \text{diag}^{-1}(\pi_{t+1}) \mathbf{y}_{t+1} - \mathbf{s}_{t+1}\notag \\
& 
= \text{diag}^{-1}(\pi_{t+1}) B_t \mathbf{y}_t - \mathbf{s}_t 
+ \text{diag}^{-1}(\pi_{t+1}) B_t (\mathbf{z}_{t+1} - \mathbf{z}_t) \notag \\
& \quad- \mathbf{1}_n \mathbf{1}_n^\top (\mathbf{z}_{t+1} - \mathbf{z}_t).
\end{align*}

\noindent Define \(r_1 = \text{diag}^{-1}(\pi_{t+1}) B_t \mathbf{y}_t - \mathbf{s}_t\), and \(r_2 = \text{diag}^{-1}(\pi_{t+1}) B_t (\mathbf{z}_{t+1} - \mathbf{z}_t) - \mathbf{1}_n \mathbf{1}_n^\top (\mathbf{z}_{t+1} - \mathbf{z}_t)\). We analyze \(r_1\) and \(r_2\) separately.

For \(r_1\), we have
\begin{align*}
\|r_1\|_{\pi_{t+1}}^2 
& = \sum_{i=1}^n [\pi_{t+1}]_i \left\| \frac{\sum_{j=1}^n [B_t]_{ij} y_{j,t}}{[\pi_{t+1}]_i} - \sum_{j=1}^{n} y_{j,t} \right\|^2 \notag \\
&
\leq \tau_t^2 \sum_{i=1}^n [\pi_t]_i \left\| \frac{y_{i,t}}{[\pi_t]_i} - \sum_{j=1}^{n} y_{j,t} \right\|^2 
\notag \\
& = \tau_t^2 S^2(\mathbf{y}_t, \pi_t),
\end{align*}
where the inequality is based on Lemma~\ref{le:pi}, by taking \(\mathcal{G} = \mathcal{G}_t\), \(B = B_t\), \(\pi = \pi_{t+1}\), and \(\nu = \pi_t\), together with the definition of \(\tau_t\). 

Taking conditional expectation and applying \(\tau_t \leq \tau\), we obtain
\begin{equation}
\label{ytpi1}
\mathbb{E} \left[ \|r_1\|_{\pi_{t+1}}^2 \right] \leq \tau^2 \mathbb{E} \left[ S^2(\mathbf{y}_t, \pi_t) \right].
\end{equation}
For \(r_2\), we define \(\Delta \mathbf{z}_t = \mathbf{z}_{t+1} - \mathbf{z}_t\) and $\tilde{\Delta}= \sum_{j=1}^n \Delta z_{j,t}$, then
\begin{align*}
\left\| r_2 \right\|_{\pi_{t+1}}^2 
&\leq \tau_t^2 \sum_{i=1}^n [\pi_t]_i \left\| \frac{\Delta z_{i,t}}{[\pi_t]_i} - \sum_{j=1}^n \Delta z_{j,t} \right\|^2 \\
&=\tau_t^2 \sum_{i=1}^{n} \pi_i \left( \left\| \frac{\Delta z_i}{[\pi_t]_i} \right\|^2 - 2 \left\langle \frac{\Delta z_i}{[\pi_t]_i}, \tilde{\Delta} \right\rangle + \|\tilde{\Delta}\|^2 \right)\\
&=\tau_t^2 \sum_{i=1}^{n} \frac{1}{[\pi_t]_i}\|\Delta z_{i,t}\|^2 - \|\tilde{\Delta}\|^2\\
&\leq \tau_t^2 \kappa_t^2 \left\| \Delta \mathbf{z}_t \right\|^2,
\end{align*}
where $\kappa_t$ is defined in \eqref{definecon}. 
Then, applying Lemma~\ref{lemmazeta}, it can be derived that
\begin{align}
\label{ytpi3}
& \mathbb{E} \left[ S^2(\mathbf{y}_{t+1}, \pi_{t+1}) \right] \notag \\
& \leq \zeta \tau^2 \mathbb{E} \left[ S^2(\mathbf{y}_t, \pi_t) \right] 
+ \frac{\zeta}{\zeta - 1} \tau^2  \kappa_t^2 \mathbb{E} \left[ \|\mathbf{z}_{t+1} - \mathbf{z}_t\|^2 \right].
\end{align}
\noindent Choosing \(\zeta = \frac{1}{\tau} > 1\) and substituting into \eqref{ytpi3} yields the desired result.
\end{proof}

\section{Proof of Lemma \ref{zt1xt1}}

\begin{proof}
Define the stochastic gradient noise at agent \(i\) and time \(t+1\) as \(\delta^1_{i,t+1} = \nabla \hat{f}_{i,t+1}(x_{i,t+1}, \xi_{i,t+1}) - \nabla f_{i,t+1}(x_{i,t+1})\), and an auxiliary noise term \(\delta^2_{i,t} = \nabla \hat{f}_{i,t+1}(x_{i,t}, \xi_{i,t+1}) - \nabla f_{i,t}(x_{i,t})\), where the randomness is induced by \(\xi_{i,t+1}\). 
Note that $\mathbb{E}[\delta^1_{i,t+1}]=0 $ but $\mathbb{E}[\delta^2_{i,t}]\neq 0$ generally due to the time-varying objective functions. 

Let $\mathbf{\delta}^1_t=[\delta^1_{i,t}]_{i\in \mathcal{V}}$ and $\mathbf{\delta}^2_t=[\delta^2_{i,t}]_{i\in \mathcal{V}}$. It can be derived that
\vspace{-1.5em}  
\begin{small} 
\begin{align}
& \mathbb{E} \left[ \| \mathbf{z}_{t+1} - \nabla F_{t+1}(\mathbf{x}_{t+1}) \|^2 \right]\notag \notag \\
&= \mathbb{E} \left[ \|  
    \beta \mathbf{\delta}^1_{t+1} 
    + (1 - \beta)(\mathbf{\delta}^1_{t+1} - \mathbf{\delta}^2_{t}) + (1 - \beta) \left( \mathbf{z}_{t} - \nabla F_{t}(\mathbf{x}_{t}) \right) 
\|^2 \right]\notag  \\
&\leq 2\beta^2\mathbb{E} \left[ \|  
     \mathbf{\delta}^1_{t+1} \|^2 \right] +2(1 - \beta)^2\mathbb{E} \left[ \|\mathbf{\delta}^1_{t+1} - \mathbf{\delta}^2_{t}\|^2\right] \notag \\
& \quad -2(1 - \beta)^2\langle \nabla F_{t+1}(\mathbf{x}_{t})-\nabla F_{t}(\mathbf{x}_{t}), \mathbf{z}_{t} - \nabla F_{t}(\mathbf{x}_{t})\rangle \notag \\
&  \quad +  (1 - \beta)^2\mathbb{E} \left[  \|\mathbf{z}_{t} - \nabla F_{t}(\mathbf{x}_{t})\|^2\right].\label{E-Zf2}
\end{align}
\end{small}

Moreover, for any $\zeta_0>0$, we have
\begin{align}
 &   -2\langle \nabla F_{t+1}(\mathbf{x}_{t})-\nabla F_{t}(\mathbf{x}_{t}), \mathbf{z}_{t} - \nabla F_{t}(\mathbf{x}_{t})\rangle \notag \\
 & \leq \zeta_0 \|\mathbf{z}_{t} - \nabla F_{t}(\mathbf{x}_{t}\|^2 +\zeta_0^{-1}\|F_{t+1}(\mathbf{x}_{t})-\nabla F_{t}(\mathbf{x}_{t})\|^2 \notag  \\
 & \leq \zeta_0 \|\mathbf{z}_{t} - \nabla F_{t}(\mathbf{x}_{t}\|^2+\zeta_0^{-1}nq_t^2. \label{cross-term}
\end{align}

By applying Assumptions \ref{ass:Lg} and \ref{ass:sigma}, we have $ \mathbb{E} \left[ \| \delta^1_{i,t+1} \|^2 \right]\leq \sigma^2$ and 
\begin{align*}
    &\mathbb{E} \left[ \| \delta^1_{i,t+1} - \delta^2_{i,t} \|^2 \right]  \\
    \leq&  2\mathbb{E}\left[ \|\nabla \hat{f}_{i,t+1}(x_{i,t+1}, \xi_{i,t+1}) - \nabla \hat{f}_{i,t+1}(x_{i,t}, \xi_{i,t+1}) \|^2 \right] \\
    &+ 2 \|\nabla f_{i,t+1}(x_{i,t+1}) - \nabla f_{i,t}(x_{i,t}) \|^2 \\
    \leq&  2\mathbb{E}\left[ \|\nabla \hat{f}_{i,t+1}(x_{i,t+1}, \xi_{i,t+1}) - \nabla \hat{f}_{i,t+1}(x_{i,t}, \xi_{i,t+1}) \|^2 \right] \\
    &+ 4 \|\nabla f_{i,t+1}(x_{i,t+1}) - \nabla f_{i,t+1}(x_{i,t}) \|^2 \\
    &+ 4 \|\nabla f_{i,t+1}(x_{i,t}) - \nabla f_{i,t}(x_{i,t}) \|^2 \\
    \leq & 6L_g^2\|x_{i,t+1}-x_{i,t}\|^2+4 q_t^2,  
\end{align*}
which implies that 
\begin{align}\label{delta-diff}
\mathbb{E} \left[ \|\mathbf{\delta}^1_{t+1} - \mathbf{\delta}^2_{t}\|^2 \right] \leq  6L_g^2\|\bm{x}_{t+1}-\bm{x}_{t}\|^2+4 nq_t^2. 
\end{align}
Then, substituting \eqref{cross-term} and \eqref{delta-diff} into \eqref{E-Zf2} results in \eqref{z_bound}. 
\end{proof}

\section{Proof of Corollary \ref{cor1}}
\begin{proof}
When $f_t= f$, the previous Lemmas \ref{zt1zt} and \ref{zt1xt1} related to the time-varying term $q_t$ can be revised as follows. Following the proof of Lemma \ref{zt1zt}, we have
\begin{align*}
\|z_{i,t+1} - z_{i,t}\|^2 &= \|\nabla \hat{f}_{i}(x_{i,t+1}, \xi_{i,t+1}) - \nabla \hat{f}_{i}(x_{i,t}, \xi_{i,t+1})  \notag \\
& \quad + \beta (\nabla \hat{f}_{i}(x_{i,t}, \xi_{i,t+1}) - \nabla f_{i}(x_{i,t})) \notag \\
& \quad - \beta (z_{i,t} - \nabla f_{i}(x_{i,t}))\|^2\\
&\leq 2 \| \nabla \hat{f}_{i}(x_{i,t+1}, \xi_{i,t+1}) - \nabla \hat{f}_{i}(x_{i,t}, \xi_{i,t+1})\|^2 \\
& \quad + 2\beta^2\| \nabla \hat{f}_{i}(x_{i,t}, \xi_{i,t+1})- \nabla f_{i}(x_{i,t})) \|^2 +\notag \\
& \quad + 2\beta^2\| z_{i,t} - \nabla f_{i}(x_{i,t})\|^2\\
&\leq 2 L_g^2\| x_{i,t+1} -x_{i,t}\|^2+ 2\beta^2\sigma^2  \\
& \quad + 2\beta^2\| z_{i,t} - \nabla f_{i}(x_{i,t})\|^2
\end{align*}
where the above inequalities uses Lemma \ref{lemmazeta} and Assumptions \ref{ass:Lg}, \ref{ass:sigma}. Hence, we obtain
\begin{align*}
\mathbb{E} \left[ \| \mathbf{z}_{t+1} - \mathbf{z}_{t} \|^2  \right] &\leq 2 L_g^2 \mathbb{E} \left[ \| \mathbf{x}_{t+1} - \mathbf{x}_{t} \|^2 \right] + 2 n\beta^2 \sigma^2 \notag\\
& \quad + 2 \beta^2 \mathbb{E} \left[ \| \mathbf{z}_{t} - \nabla F_{t}(\mathbf{x}_{t}) \|^2 \right].
\end{align*}

For Lemma \ref{zt1xt1}, we define 
\(\delta^1_{i,t+1} = \nabla \hat{f}_{i}(x_{i,t+1}, \xi_{i,t+1}) - \nabla f_{i}(x_{i,t+1})\) and \(\delta^2_{i,t} = \nabla \hat{f}_{i}(x_{i,t}, \xi_{i,t+1}) - \nabla f_{i}(x_{i,t})\). Then, one can reorganize \eqref{E-Zf} as
\vspace{-1.5em}  
\begin{small} 
\begin{align}
& \mathbb{E} \left[ \| \mathbf{z}_{t+1} - \nabla F_{t+1}(\mathbf{x}_{t+1}) \|^2 \right]\notag \notag \\
&= \mathbb{E} \left[ \|  
    \beta \mathbf{\delta}^1_{t+1} 
    + (1 - \beta)(\mathbf{\delta}^1_{t+1} - \mathbf{\delta}^2_{t}) + (1 - \beta) \left( \mathbf{z}_{t} - \nabla F_{t}(\mathbf{x}_{t}) \right) 
\|^2 \right]\notag  \\
&\leq 2\beta^2\mathbb{E} \left[ \|  
     \mathbf{\delta}^1_{t+1} \|^2 \right] +2(1 - \beta)^2\mathbb{E} \left[ \|\mathbf{\delta}^1_{t+1} - \mathbf{\delta}^2_{t}\|\right] \notag \\
&  \quad +  (1 - \beta)^2\mathbb{E} \left[  \|\mathbf{z}_{t} - \nabla F_{t}(\mathbf{x}_{t})\|^2\right]\notag \\
& \leq (1 - \beta)^2 \mathbb{E} \left[ \| \mathbf{z}_t - \nabla F_t(\mathbf{x}_t) \|^2 \right] 
+ 2n \beta^2 \sigma^2 \notag \\
& \quad  + 2 (1 - \beta)^2 L_g^2 \mathbb{E} \left[ \| \mathbf{x}_{t+1} - \mathbf{x}_t \|^2 \right], \label{E-Zf}
\end{align}
\end{small}
where the first inequality holds due to $\mathbb{E}[\mathbf{\delta}^1_t]= \mathbb{E}[\mathbf{\delta}^2_t]=0$, and the second inequality is obtained by applying $\mathbb{E}[\|\xi-\mathbb{E}[\xi]\|^2] = \mathbb{E}[\|\xi\|^2] - \| \mathbb{E}[\xi]^2\|$ and Assumption \ref{ass:Lg}. 

With these modifications, one can derive a new positive matrix $\widehat{M}(\alpha)\leq M(\alpha)$ element-wise, sharing the same structure as $M(\alpha)$ but with slightly different number coefficients and $m_0=(1-\beta)^2$. In this case, the following inequality system holds
\begin{align}
\label{vt3}
V_{t+1} \leq M(\alpha)V_t + b,
\end{align}
with $b= [0,\frac{2 n\tau^2 \psi}{1 - \tau}\beta^2  \sigma^2,0,2n \beta^2 \sigma^2]^{\top}$. By iteratively expanding this inequality, we get
\begin{align*}
V_{t+1}\leq M(\alpha)^t V_0 + \sum_{k=0}^{t-1} M(\alpha)^k b.
\end{align*}

Since the spectral radius $\rho(M(\alpha)) < 1$, we have $\lim_{t \to \infty} M(\alpha)^t = 0$. Therefore, the first term $M(\alpha)^t V_0$ tends to zero as $t \to \infty$ with a linear decay rate of $\rho_M$. Next, consider the sum $\sum_{k=0}^{t-1} M(\alpha)^k b$, which is a geometric series that can be written as
$$
\sum_{k=0}^{t-1} M(\alpha)^k b = (\mathbb{I} - M(\alpha))^{-1}( \mathbb{I}-M(\alpha)^t)b.
$$
As $t \to \infty$, $M(\alpha)^t \to 0$, so the above expression simplifies to
$$
\sum_{k=0}^{\infty} M(\alpha)^k b= -(\mathbb{I} - M(\alpha))^{-1}b.
$$
Therefore, when $t \to \infty$, 
$\limsup_{t \to \infty} V_t \leq -(\mathbb{I}- M(\alpha))^{-1}b.
$ with a linear convergence rate of $\rho_M$.
\end{proof}

\bibliographystyle{plain}
\bibliography{AUTO2025}

\end{document}